%% file: neurips_2023.tex
\documentclass[letterpaper]{article}
\usepackage{aaai24}
\usepackage{times}  
\usepackage{helvet}  
\usepackage{courier}  
\usepackage[hyphens]{url}  
\usepackage{graphicx} 
\usepackage{natbib}  
\usepackage{caption} 
\usepackage[utf8]{inputenc} 
\usepackage[T1]{fontenc}    
\usepackage{url}            
\usepackage{booktabs}       
\usepackage{amsfonts}       
\usepackage{nicefrac}       
\usepackage{microtype}      
\usepackage{xcolor}         
\usepackage{amsmath,amssymb}
\usepackage{bm}
\usepackage{graphicx}

\usepackage{caption}
\usepackage{subcaption}

\usepackage{booktabs}
\usepackage{multirow}
\usepackage{enumitem}

\usepackage{amsthm}
\newtheorem{theorem}{Theorem}[section]
\newtheorem{lemma}[theorem]{Lemma}

\usepackage{algorithm}
\usepackage{algpseudocode}

\newcommand{\eat}[1]{}

\newcommand{\naheed}[1]{{\color{black}#1}}
\newcommand{\jess}[1]{{\leavevmode\color{black}#1}}
\newcommand{\yli}[1]{{\color{black}#1}}

\title{Finite Volume Features, Global Geometry Representations, and Residual Training for Deep Learning-based CFD Simulation
}

\author {
    Loh Sher En Jessica \textsuperscript{\rm 1,\rm 3}\equalcontrib,
    Naheed Anjum Arafat\textsuperscript{\rm 1}\equalcontrib,
    Wei Xian Lim\textsuperscript{\rm 2},
    Wai Lee Chan\textsuperscript{\rm 2},
    Adams Wai Kin Kong\textsuperscript{\rm 1,\rm 3}
}
\affiliations {
    \textsuperscript{\rm 1}Rolls-Royce@NTU Corporate Lab, 
\textsuperscript{\rm 2}School of Mechanical \& Aerospace Engineering,  \textsuperscript{\rm 3} School of Computer Science \& Engineering, \\
Nanyang Technological University, Singapore \\
\{lohs0037,weixian001\}@e.ntu.edu.sg, naheed\_anjum@u.nus.edu,\{chan.wl, adamskong\}@ntu.edu.sg
}
\begin{document}
\maketitle

\begin{abstract}
  Computational fluid dynamics (CFD) simulation is an irreplaceable modelling step in many engineering designs, but it is often computationally expensive. Some graph neural network (GNN)-based CFD methods have been proposed. However, the current methods inherit the weakness of traditional numerical simulators, as well as ignore the cell characteristics in the mesh used in the finite volume method, a common method in practical CFD applications. Specifically, the input nodes in these GNN methods have very limited information about any object immersed in the simulation domain and its surrounding environment. Also, the cell characteristics of the mesh such as cell volume, face surface area, and face centroid are not included in the message-passing operations in the GNN methods. To address these weaknesses, this work proposes two novel geometric representations: Shortest Vector (SV) and Directional Integrated Distance (DID). Extracted from the mesh, the SV and DID provide global geometry perspective to each input node, thus removing the need to collect this information through message-passing. This work also introduces the use of Finite Volume Features (FVF) in the graph convolutions as node and edge attributes, enabling its message-passing operations to adjust to different nodes. Finally, this work is the first to demonstrate how residual training, with the availability of low-resolution data, can be adopted to improve the flow field prediction accuracy. Experimental results on two datasets with \jess{five} different state-of-the-art GNN methods for CFD indicate that \jess{SV, DID, FVF and residual training can effectively reduce the predictive error of current GNN-based methods by as much as 41\%}.
\end{abstract}

%

\input{1_intro}
\input{2_prelim}
\input{3_method}

\input{4_expt}
\input{5_conc}
\bibliography{mybibfile}
\input{6_appendix}


\end{document}

%% file: 1_intro.tex
\section{Introduction}
\label{sec:intro}
Computational fluid dynamics (CFD) is a branch of fluid dynamics in which physical phenomena involving fluid flow are modelled mathematically as partial differential equations (PDEs), like the Navier–Stokes (NS) equations, and solved computationally via numerical analysis. CFD is applied to a wide range of scientific and engineering problems that requires the flow of the fluid and its interaction with surfaces to be simulated, including aircraft aerodynamic optimisation~\citep{martins2022}, combustion engine design~\citep{vijashree2018}, marine hydrodynamics prediction~\citep{demirel2021}, microfluidic device evaluation~\citep{chaves2020}, and urban planning~\citep{zhang2021}. Despite its versatility, CFD simulation is generally slow and/or costly due to the need for both high spatial and temporal resolutions 
\yli{to solve the governing PDEs accurately.}

Researchers have exploited deep learning to accelerate CFD simulation. Multilayer perceptron (MLPs) and convolutional neural networks (CNNs) have been considered. However, they both do not fit industrial requirements in many cases because of their restrictions in the input fields and architectures. MLPs cannot handle high dimensional input, which would increase the number of training parameters dramatically and cause overlearning. Thus, the current MLP methods such as PINN consider each spatial location separately and ignore their relationship~\citep{raissi2019}. In other words, they do not have an explicit scheme for information exchange between nodes. It should be emphasised that flow of a fluid in a particular location would influence 
that of a neighbour region. Although CNNs allow this information exchange, they can apply only to the flow field represented on a fixed, regular grid. In many industrial simulations, the computational domain may contain objects with a complex geometry, such as a turbine blade, and hence may have an irregular mesh, so CNNs are not suitable. To fit the industrial requirements and bypass the limitations, researchers have recently employed graph neural networks (GNN) for CFD problems, such as laminar and turbulent flow prediction and geometry optimisation~\citep{liu2020,baque2018}. Some of these studies have considered physical properties or constraints in their GNN. For example, \citet{horie2022} designed boundary encoders to impose Dirichlet boundary conditions and \citet{airfrans} separated the nodes on airfoils and other nodes in their objective for computing drag and lift coefficients. 

The current GNN methods neglect the cell characteristics in the mesh, such as cell volume, face surface area, and face centroid, which are 
\yli{core components of} the finite volume method, a widely adopted CFD method in industry~\citep{cfdfvbook}. CFD simulators based on the finite volume method in general take three steps to compute the flow of the fluid. First, the simulation domain will be discretised into a finite number of small volumes known as cells. Next, the fluid properties, flow models, boundary conditions of the domain, and initial conditions for the flow will be 
\yli{prescribed}. Finally, 
\yli{through} Gauss’s divergence theorem and \yli{assuming a constant} solution in each cell, the governing PDEs can be written in an integral form, discretised by numerical approximation, and solved as a system of algebraic equations through numerical methods. 
The cell characteristics 
play a critical role in the discretisation process because they are used to model the flux between a certain cell and its neighbours for the purpose of conservation of mass, momentum, and energy. However, the current GNN methods do not consider this information. 

Another weakness of the current GNN methods is that 
\yli{their} input node features, 
such as signed distance function (SDF), spatial coordinates, and 
inlet velocity, only provide very limited information about 
\yli{any embedded object and its} surrounding environment to the nodes~\citep{airfrans,cfdgcn}.
Thus, \yli{the nodes} need to collect this information through message-passing between \yli{neighbouring} nodes. In fact, this weakness also appears in the traditional 
\yli{finite volume simulators, which propagate object boundary condition only locally.}


To address these weaknesses, in this paper, we make the following three contributions regarding three aspects of graph neural network training for CFD:
\begin{itemize}
    \item \textbf{Input \yli{layer}.}~We propose Shortest Vector (SV) and Directional Integrated Distance (DID) for enhancing the performance of GNN-based methods. The SV and DID extracted from the mesh provide a global geometry perspective to each input node, thus removing the need to collect this information through message-passing and easing the learning.
    \item \textbf{Graph convolution.}~We propose the Finite Volume Features (FVF), including cell volume, face 
    \yli{area normal vector}, and face centroid, to be used as node and edge attributes in graph convolution, such that the convolution filters can be adjusted based on the cell characteristics.
    \yli{A theorem is presented to show} that the input mesh can be reconstructed from the 
    FVF.
    \item \textbf{Training scheme.}~While existing methods~\citep{cfdgcn} exploit low-resolution data as prior knowledge, we demonstrate that residual training reduces the prediction error by helping the model focus more on regions where the low-resolution data tends to be less accurate. 

\end{itemize}	
The experimental results show that the (i) combined effect of the proposed geometric features and finite volume features 
\jess{reduces predictive errors of MeshGraphNet~\citep{deepmind2020}, BSMS-GNN~\citep{bsmsgnn2022}, Chen-GCNN~\citep{chen2021} and Graph U-Net~\citep{airfrans} by as much as $41\%$, as well as reduces the predictive error of CFDGCN~\citep{cfdgcn} by about $24\%$, and (ii) additional usage of residual training increases the reduction of the error of CFDGCN to $41\%$}.

%

%% file: 2_prelim.tex
\section{Preliminaries and Related Work}
\label{sec:prelim}
\paragraph{Graph construction.}In GNN-based CFD methods, both the inputs and outputs of the model are often graphs. 
\yli{The CFD simulation} mesh $M$ \yli{is represented} as a graph $G=(V,E)$, where $V$ and $E$ represent a set of nodes and edges, respectively. There are two methods to do 
\yli{so}.

\yli{Using a 2D CFD mesh for illustration,} the first method, as shown in Figure \ref{fig:meshgraph}, is to directly represent the mesh nodes as graph nodes $i\in V$, and the faces between them as edges $(i,j),(j,i)\in E$. We refer to this as the \textit{mesh node-based} method, and it has been used by other researchers~\citep{airfrans,cfdgcn,chen2021,deepmind2020}.
%
An alternate approach proposed by the authors is the following: the graph node $i\in V$ represents the mesh cell centroid $m_i \in M$, and the bi-directional edge $(i,j),(j,i)\in E$ indicates that cells $i$ and $j$ are adjacent, i.e., share a face. We refer to this approach as the \textit{cell centroid-based} method, \yli{which is illustrated in Figure~\ref{fig:cfdgraph}.} 
This method is key to the use of FVF as described in \S~\ref{sec:method}.
%
Note that in this method, the centroids of the boundary faces are represented as graph nodes as well, despite not being shared by two cells, to capture flow characteristics at the boundary.
This representation allows message-passing from and towards boundary faces to be captured by the edges between these nodes and that of the cell 
\yli{adjacent to the boundary face.}
\eat{
\begin{figure}[htb!]
     \centering
     \begin{subfigure}[b]{0.8\textwidth}
         \centering
         \includegraphics[width=\textwidth]{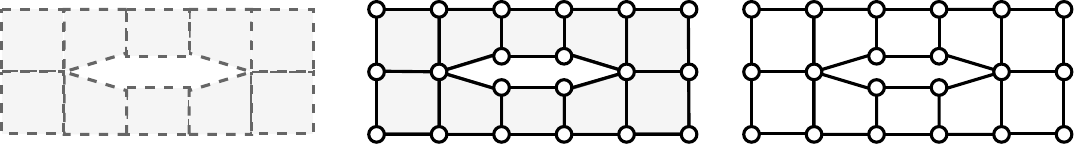}
         \caption{\textbf{Mesh node-based construction.} Graph nodes represent mesh nodes and graph edges represent mesh faces.}
         \label{fig:meshgraph}
     \end{subfigure}
     \begin{subfigure}[b]{0.8\textwidth}
         \centering
         \includegraphics[width=\textwidth]{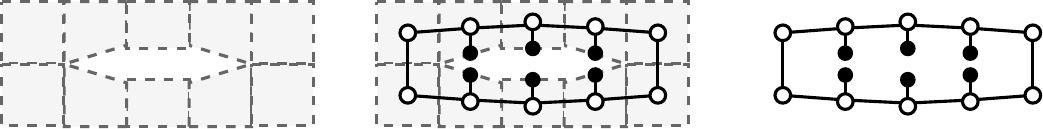}
         \caption{\textbf{Cell centroid-based constuction.} Graph nodes represent either cell centroids (white) or geometry boundary face centroids (black), and graph edges represent either internal faves or geometry boundary faces.}
         \label{fig:cfdgraph}
     \end{subfigure}
        \caption{Two construction methods to obtain the graphs used in GNN training (right) from the original CFD mesh (left).}
        \label{fig:graphcontruction}
\end{figure}
}
\begin{figure}[tb!]
     \centering
     \begin{subfigure}[b]{0.4\textwidth}
         \centering
    \includegraphics[width=1\linewidth]{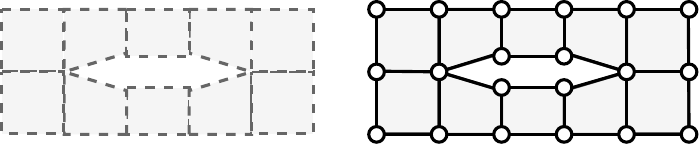}
         \caption{}
         \label{fig:meshgraph}
     \end{subfigure}
     \hfill
     \begin{subfigure}[b]{0.4\textwidth}
         \centering
         \includegraphics[width=\linewidth]{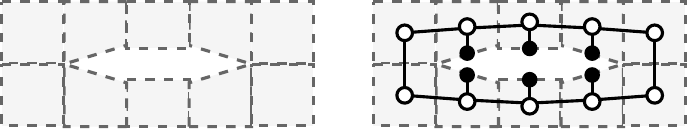}
         \caption{}
         \label{fig:cfdgraph}
     \end{subfigure}
        \caption{Two graph construction methods: (a) \textbf{Mesh node-based:} Graph nodes represent mesh nodes and graph edges represent mesh faces. (b) \textbf{Cell centroid-based:} Graph nodes represent either cell centroids (white) or  boundary face centroids (black), and graph edges represent 
        the adjacency of the cell centroids with one another or with boundary faces.}
    \label{fig:graphcontruction}
\end{figure}

%
\paragraph{Global geometry representations.}In a steady simulation, a GNN is trained to predict the velocity vector and pressure for each node. To train a GNN to predict target flow characteristics at each node, the current methods encode some features into the input nodes.
The most common approaches include variants of the binary representation~\citep{chen2021} and the Signed Distance Function (SDF)~\citep{airfrans,cfdgcn,guo2016}. 
In binary representations, nodes are given discrete values such as $0$ and $1$, depending on whether they are on the geometry boundary or not.
The SDF, proposed by \citet{guo2016} for CNNs, is defined as
\begin{equation}
SDF(\bm{x}_{i})=\min_{\bm{x}_b\in B}||\bm{x}_{i}-\bm{x}_b||h(\bm{x}_{i}),
\end{equation}
where \yli{$\bm{x}_i$ and $\bm{x}_b$ denote an internal node and its closest boundary node, respectively.}
$h(\bm{x}_{i})$ is equal to $1$, $-1$, and $0$ if $\bm{x}_{i}$ is outside, inside, and on the object boundary, respectively. 
The SDF representation was shown to be more effective than the simple binary representations in the CNN case~\citep{guo2016}.
Since graphs generated from meshes do not have nodes on the inside of the object, the SDF for GNNs is in fact just the shortest Euclidean distance between $\bm{x}_{i}$ and the object, estimated in a discrete case. However, it provides each node with very limited information about the object. The SDF value only indicates the existence of an object at the distance of $SDF(\bm{x}_i)$. No information about the object's size, shape, or direction from the node is given by the SDF, 
\yli{even though all these factors will affect the flow at the node.}
The boundary of the circle in Figure~\ref{fig:geomfeat}(a) indicates the uncertainty due to incomplete information of $SDF(\bm{x}_{i})$. Although some other properties about the nodes and the flow, such as the spatial location of the nodes, inlet, angle of attack, and Mach number~\citep{cfdgcn}, can also be provided as input node features, they cannot substitute the missing geometric information, which remains to be acquired by the nodes through message-passing. 
\yli{Note that this weakness is also true of typical CFD simulators where each mesh cell carries no global geometry information and the object boundary condition is transmitted only locally.}
\begin{figure}[htb!]
     \centering
     \begin{subfigure}[b]{0.3\columnwidth}
         \centering
         \includegraphics[width=\textwidth]{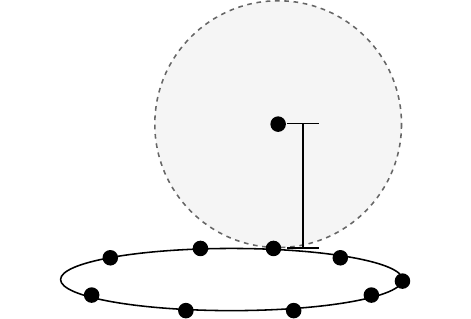}
         \caption{SDF}
         \label{fig:SDF}
     \end{subfigure}
     \begin{subfigure}[b]{0.3\columnwidth}
         \centering
         \includegraphics[width=\textwidth]{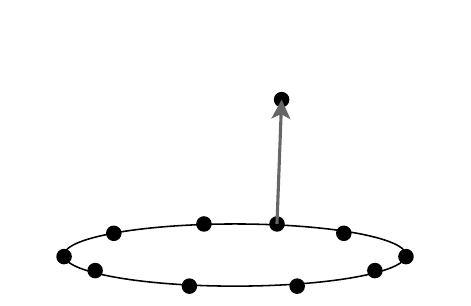}
         \caption{SV}
         \label{fig:SV}
     \end{subfigure}
     \begin{subfigure}[b]{0.3\columnwidth}
         \centering
        \includegraphics[width=\textwidth]{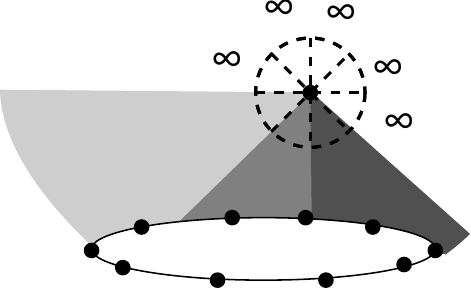}
         \caption{DID}
         \label{fig:DID}
     \end{subfigure}
        \caption{Illustration of three geometry representations\jess{: (a) SDF value only indicates presence of the closest boundary point somewhere along the circle's circumference. (b) SV provides both distance and direction from the nearest boundary point. (c) DID gives the average distance of all boundary within several difference angle ranges.}}
        \label{fig:geomfeat}
\end{figure}
%
%
\paragraph{Prior studies.}The GNN-based CFD methods employ on-shell GNNs and exploit the physical knowledge in different ways. Some of them embed physical constraints and properties into the architectures and objective functions and some others directly use numerical simulators as a part of their methods.  \citet{cfdgcn} combined a differentiable CFD simulator and graph convolution network (GCN) for fluid flow prediction.  \citet{liu2020} exploited graph attention network (GAT) for turbulent flow 
\yli{prediction without} any physics prior.  \citet{ogoke2020graph} applied GraphSAGE~\citep{graphsage} to predict drag forces around airfoils of different shapes and angles of attack under laminar flow. They demonstrated that GraphSAGE outperforms non-graph based 
\yli{MLP and CNN methods}.  \citet{battaglia2018relational} proposed a general method called graph network (GN) blocks that can handle graphs with local features such as node features, edge features, and global graph-level features. 
\citet{sanchez2020learning} used GN blocks to predict the future roll-out of physical systems of particles, including fluids, rigid solids, and deformable materials. \citet{deepmind2020} used GNN to predict future roll-outs of unsteady flow and showed better performance than CNN on various simulation scenarios.
More recently, \citet{airfrans} released a large-scale high-resolution two-dimensional (2D) Reynolds-averaged Navier-Stokes (RANS) simulation datasets on airfoils and demonstrated good generalisation capabilities of GraphSAGE~\citep{graphsage} and Graph U-Net~\citep{gunet} to different physical conditions 
\yli{and} airfoil geometries. 
DiscretizationNet~\citep{ranade2021discretizationnet} used finite volume discretisation to approximate spatio-temporal partial derivatives in its CNN encoder-decoder training. It is a non-data driven \yli{method} that cannot generalise to new scenarios, requiring it to be trained on every new instance to obtain a solution. Also, its convolutions did not leverage cell characteristics as the finite volume method does. \citet{chen2021} proposed their own permutation-invariant edge-convolution layer and smoothing layer to predict laminar flow on obstacles of different shapes. Their proposed architecture, which is referred to as \textit{Chen-GCNN} in this paper, showed better performance than the standard U-net model~\citep{unet2015}. All these GNN works neglect the cell characteristics, which, 
in the finite volume method, 
are used to discretise the governing equations and model the flow among the cells. 

%% file: 3_method.tex
\section{Proposed Features and Training Scheme}
\label{sec:method}
\subsection{Global Geometry Perspective}
To give each node global geometric perspective about the object and its surrounding environment, we propose two features: Shortest Vector (SV), and Directional Integrated Distance (DID). SV is the vector formed by
an internal cell centroid node $\bm{x}_{i}$ and its closest boundary face node $\bm{x}_{b}$ (Figure~\ref{fig:SV}), and
is defined as 
\begin{equation}
\label{eqn:sv}
\phi(\bm{x}_{i}) = \bm{x}_{i}-\bm{x}_b\text{ s.t. }\min_{\bm{x}_b\in B}||\bm{x}_{i}-\bm{x}_b|| \;.
\end{equation}
SV is related to SDF in that the length of SV, $||\phi\left(\bm{x}_{i})||= SDF(\bm{x}_{i}\right)$. However, SV has more discriminative power than SDF because two nodes, $\bm{x}_{i}$ and $\bm{x}_{j}$, can have the same SDF values (i.e., $SDF(\bm{x}_{i} )=SDF(\bm{x}_{j})$) despite the vectors $\phi(\bm{x}_{i})$ and $\phi(\bm{x}_{j})$ not being the same. 

One limitation of SV is that only the closest node on the object's boundary 
is represented. Information in other directions is still missing. DID defined on an angular range (Figure~\ref{fig:DID} and Figure~\ref{fig:did}) is a generalisation of SV. DID has different angular segments to handle different directions. Figure~\ref{fig:did} shows four overlapping angular segments, each with a range of $2\pi/3$. To clearly present the concept of DID in the following description, we consider a continuous object boundary. 
For the $j^{th}$ segment of DID, $\theta_j$ and $\theta'_{j}$ represent the starting and ending angles of the segment, respectively, where $\theta_j\leq\theta'_{j}$. The continuous version of DID is defined as 
\begin{equation}
DID(\bm{x}_{i},\theta_j,\theta'_{j},B_c)=\int_{\theta_j}^{\theta'_{j}}w_j(\theta)g(\bm{x}_{i},B_c,\theta)d\theta \;,
\end{equation}
where $w_j(\theta)$ is a suitable weightage function such as a Gaussian function centred at $(\theta_j+\theta'_{j})/2$, $B_c$ is the continuous boundary of the object, \jess{and $g(\bm{x}_{i},B_c,\theta)$ is the 
distance between the object and a node $\bm{x}_{i}$ in the direction $\theta$.} If there is no object boundary at the direction $\theta$, an appropriately large constant, as denoted in Figure \ref{fig:DID} as $\infty$, is given to $g(\bm{x}_{i},B_c,\theta)$.
In this continuous version of DID, $B_c$ is represented by a parametric equation rather than a set of nodes.
Each DID value provides a weighted average distance between the object and the node in a particular angle range.
While shorter angular segments make a more accurate description of the object, it increases the number of input features to train the network with.
In this work, we use SV and DID to provide \naheed{relative} global geometry information of the object and its environment to the input nodes. Hence, this information will not have to be collected during propagation, easing the learning process.
The algorithmic details of our discrete, non-parametric DID implementation can be found in the appendix.
%
\subsection{Finite Volume Graph Convolution} 
\label{subsec:fvf}
\yli{
Consider the integral form of the steady state, incompressible turbulent Navier--Stokes equation for $x$ direction over a control volume,
which is given by
\begin{equation} \label{eqn:nse}
\begin{aligned}
    &\int \nabla\cdot\left(\bar{u}_x\overline{\bm{U}}-(\nu+\nu_t)\nabla\bar{u}_x\right)
    + \frac{\partial\bar{p}}{\partial{x}} \mathrm{d}\text{V} \\
     &=  \oint \left(\bar{u}_x\overline{\bm{U}} - (\nu+\nu_t)\nabla\bar{u}_x\right)\cdot\mathrm{d}\bm{S} + \int \frac{\partial\bar{p}}{\partial{x}} \mathrm{d}\text{V} = 0 \;,
\end{aligned}
\end{equation}
where $\overline{\bm{U}}=[\bar{u}_x,\bar{u}_y]^T$ is the velocity vector,
$\nu$ and $\nu_t$ are the dynamic and turbulent viscosities respectively, and 
$\bar{p}$ is the \naheed{normalised} pressure. 
The second expression of Equation~\ref{eqn:nse} is obtained from Gauss's theorem, 
where $\textbf{S}$ denotes the area normal vector on the surface of the control volume and 
points outwards by convention~\citep{cfdfvbook}. 
The finite volume method \naheed{discretises} the control volume into cells, 
as shown in Figure~\ref{fig:fvf}(a). 
For sufficiently small cell volume $\text{V}$ like in Figure~\ref{fig:fvf}(b), 
all variables within a cell or along each of its faces are approximately constant~\citep{cfdfvbook}, 
so Equation~\ref{eqn:nse} becomes
\begin{equation}\label{eqn:dnse}
\begin{aligned}
    &\sum_{f}\left(\bar{u}_x\overline{\bm{U}}-(\nu+\nu_t)\nabla\bar{u}_x\right)_f\cdot\bm{S}_f + \left(\frac{\partial\bar{p}}{\partial{x}}\right)\text{V} \\ 
    &= \sum_{f}\bm{\Phi}_f\cdot\bm{S}_f + \Omega\text{V} = 0 \;,
\end{aligned}
\end{equation}
where $f$ is a counter for the discrete faces of the cell. 
The first and second terms of the first expression of Equation~\ref{eqn:dnse} are known as the flux ($\bm{\Phi}$) and source ($\Omega$) terms, respectively~\citep{cfdfvbook}, 
thus constituting the second expression.

In typical finite volume simulators, variables like $\overline{\bm{U}}$ and $\bar{p}$ are solved at the cell centroids, 
such as $\bm{x}_i$ and $\bm{x}_j$ in Figure~\ref{fig:fvf}(b). 
However, from Equation~\ref{eqn:dnse}, the flux term needs face centroid values, 
which can be approximately interpolated from the cell centroids by~\citep{cfdfvpaper} 
\begin{equation}
 \label{eqn:interpl}
 \bm{\Phi}_f \approx \left(\bm{\Phi}_{i}\frac{||\bm{c}_{f,ij}-\bm{x}_i||}{||\bm{x}_{j}-\bm{x}_i||} + \bm{\Phi}_{j}\frac{||\bm{c}_{f,ij}-\bm{x}_j||}{||\bm{x}_{j}-\bm{x}_i||}\right) \;.
\end{equation}
The interpolated flux term will then be projected to its respective face area normal vector, $\bm{S}_{f,ij}$.
Note that Equation~\ref{eqn:interpl} implies that the face centroid at $\bm{c}_{f,ij}$ lies along the line connecting $\bm{x}_i$ and $\bm{x}_j$, 
which is not necessarily true because cells in a CFD simulation need not be regular and can have different sizes and shapes, as shown in Figure~\ref{fig:fvf}(b). 
In general, cells with small volumes are used in sensitive regions,
for instance the close vicinity of an object or wake region where flow variables may change drastically due to boundary conditions. 
Therefore, Equation~\ref{eqn:interpl} will incur \naheed{an error} that corresponds to the deviation from the face centroid, 
though a spatial correction scheme can be implemented as mitigation~\citep{cfdfvpaper}. 
Also, the error will reduce with smaller cell volume, which is the main reason for the high computational cost of CFD.
On the other hand, the source term weighted by the cell volume will be represented exactly by $\Omega_j\text{V}_j$.
}
\begin{figure}[tb!]
\centering 
\includegraphics[width=\linewidth]{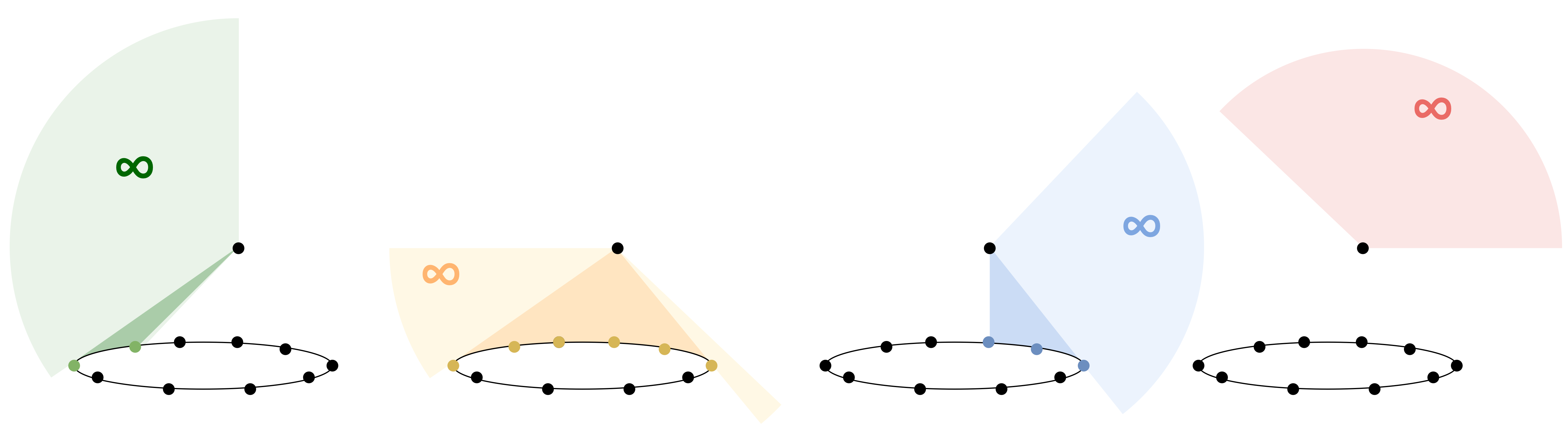}
\caption{\jess{Illustration of the computation of DID. Each angle range is represented by the weighted average of the distance to the boundary (shown by the darker segments) and the $\infty$ value for the directions where there is no boundary (shown by the lighter segments). Note that $\infty$ is a predefined number in the implementation.}}
\label{fig:did} 
\end{figure}
\begin{figure}[htb!]
	\centering
	\includegraphics[width=\linewidth,keepaspectratio]{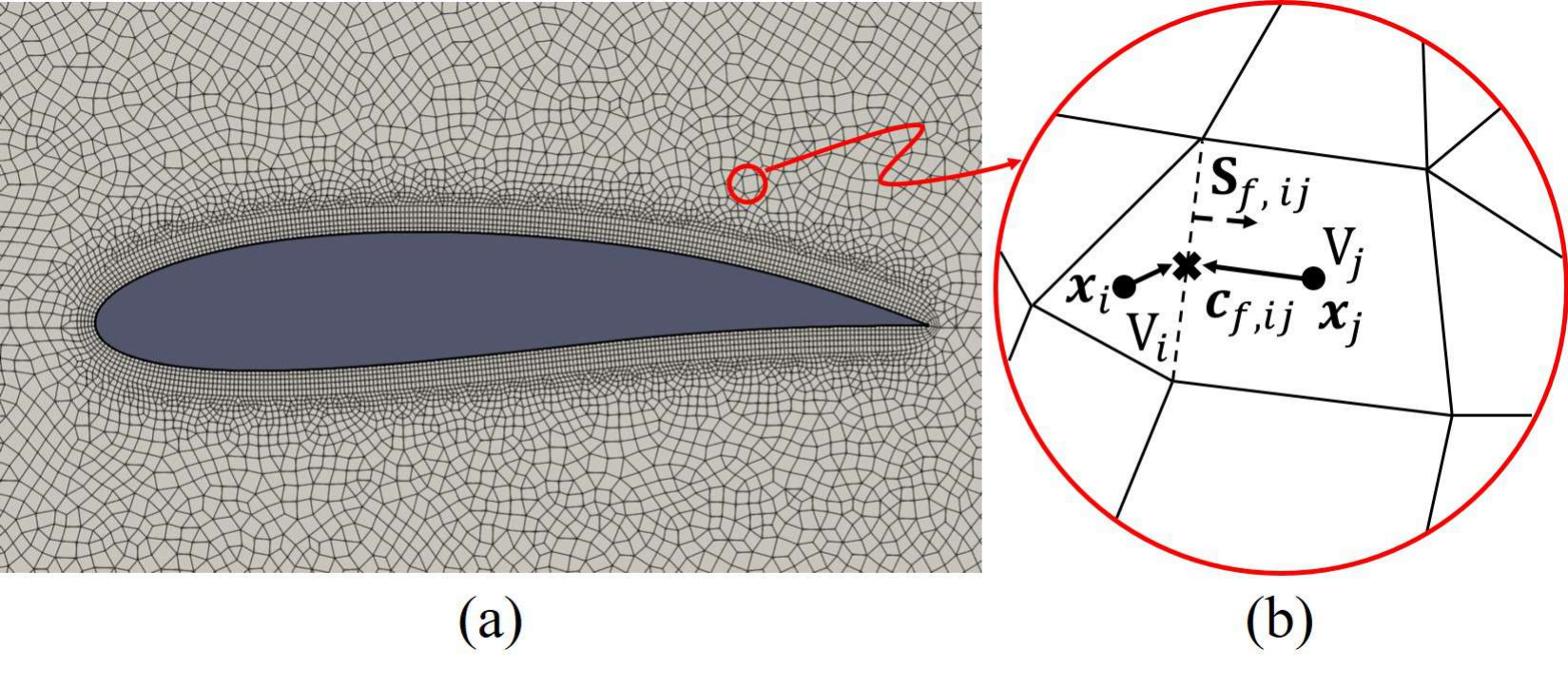} 
	\caption{(a) CFD mesh with \yli{an airfoil body surrounded by} different sizes of cells.
     (b) Illustration of \yli{cell characteristics, namely cell centroids, $\bm{x}_i$ and $\bm{x}_j$, face centroid, $\bm{c}_{f,ij}$, 
     face area normal vector, $\bm{S}_{f,ij}$, and cell volumes, $\text{V}_i$ and $\text{V}_j$.}}
\label{fig:fvf}
\end{figure}

As shown in Equations~\ref{eqn:dnse}--\ref{eqn:interpl}, the finite volume method uses cell characteristics, such as the cell and face centroids, face area normal vector, and cell volume, extensively. Motivated by the finite volume method, we hence embed these characteristics in GNN.
\yli{
To ease the following discussion, we will take the spatial graph convolution network (SGCN) as an example.}
First introduced and evaluated on graphs from chemical compounds~\citep{danel2020}, the SGCN \yli{has} two key features: 
\yli{(i)}  use of the spatial location of graph nodes as node attributes and 
\yli{(ii)} use of multiple filters.
One limitation is that the node attributes only consider node positions and ignore other useful information.
Hence, we \naheed{generalise} the convolution to be able to take \yli{in} node attributes, $p_j$, 
\yli{and} edge attributes, $q_{ij}$.
Given node feature $h_i\in \mathcal{R}^{d_{in}}$ at the $i^{th}$ node and its spatial location $\bm{x}_i\in \mathbb{R}^t$, a convolution using FVF is defined as
\begin{equation}
    \label{eqn:fvgcn}
	\Bar{h}_i(U,b,N_i)=\sum_{j\in N_i}ReLU(U^T(q_{ij})+b)\odot (h_j\oplus p_j) \;,
\end{equation}
where $U\in\mathbb{R}^{3t\times (d_{in}+1)}$ and $b\in\mathbb{R}^{d_{in}+1}$ are trainable parameters,
\yli{$t$ is the spatial dimension of the CFD simulation, and} $N_i$ is an index set indicating the \yli{neighbourhood} of the node $i$ and $\odot$ is the element-wise multiplication~\citep{danel2020}.
\naheed{In our model, the 
\yli{node attributes} are its associated cell volume, denoted as $p_j=\text{V}_j \in\mathbb{R}^{1}$,} and the edge attributes are its associated face area normal vector and the relative spatial location of its face centroid to the nodes, denoted as $q_{ij}= \textbf{S}_{f,ij} \oplus (\bm{c}_{f,ij} - \bm{x}_{i}) \oplus (\bm{c}_{f,ij} - \bm{x}_j) \in \mathbb{R}^{3t}$.
Note that cell centroids $\bm{x}_{i}$ and $\bm{x}_j$ are used as reference points when we use face centroid, 
\yli{$\bm{c}_{f,ij}$}, as with the finite volume method.
Finally, as with CNN operations, multiple filters are used and their outputs are concatenated such as
\begin{equation}
    \label{eqn:concatenation}
    \hat{h}_i(\theta,N_i,k)=\Bar{h}_i(U_1,b_1,N_i)\oplus\cdots\oplus\Bar{h}_i(U_k,b_k,N_i) \;,
\end{equation}
where $\theta=\{U_1,\cdots,U_k,b_1,\cdots,b_k\}$ are trainable parameters and $\oplus$ represents an operator of concatenation. Finally, an MLP is applied on $\hat{h}_i$ and the final output of node $i$ is obtained, whose dimension is the same as the output dimension of the MLP.

When used directly, we refer to this \yli{method} as \textit{Finite Volume Graph Convolution (FVGC)} and $\text{V}_i$, $\textbf{S}_{f,ij}$, $(\bm{c}_{f,ij}-\bm{x}_{i})$, $(\bm{c}_{f,ij} - \bm{x}_j)$
as \textit{Finite Volume Features} (FVF). 
Alternatively, the same principles can be incorporated into other graph convolution types indirectly, for instance by using the FVGC as the aggregation function of the SAGE convolution. In convolutions like the invariant edge convolution~\citep{chen2021}, 
\jess{which already employ multiple filters and edge features, just the use of FVF as node and edge attributes in each convolution has to be implemented.}

Any common 2D mesh typically used for CFD simulations, such as those with triangular or quadrilateral cells, can be reconstructed from its prescribed FVF. More specifically, 
\begin{theorem} [The Completeness of FVF]
    Let $M$ be a 2D mesh such that the cells along its farfield have no more than $2$ boundary faces each, and the faces of each of its cells enclose a singular volume.
    If a graph $G=(V,E)$ is the cell centroid-based graph representation of $M$, the relative positions of all the nodes and faces of $M$ can be uniquely deduced given {\normalfont $\textrm{V}_i$, } $\forall i\in V$, and $\textbf{S}_{f,ij}$, $(\bm{c}_{f,ij}-\bm{x}_{i})$ and $(\bm{c}_{f,ij} - \bm{x}_j)$, $\forall(i,j)\in E$.
    \label{theorem1}
\end{theorem}
The implementation details of the FVF and proof of this theorem are provided in the appendix.

\subsection{Residual Training}
\label{sec:res}
Residual training is a well-known approach in image super-resolution~\citep{zhang2018residual,yang2019deep}. The general idea is to train the network to predict the residual field $F - Upsample(F_{LR})$, \yli{where $F_{LR}$ is the low-resolution field}, instead of \yli{the original field} $F$ itself.
\yli{However,} the most common way of utilizing low-resolution data as reference in CFD--AI literature is to concatenate $ Upsample(F_{LR})$ to one of the intermediate convolution layers~\citep{cfdgcn}.
The prior knowledge from the low-resolution flow field could be further utilised through the residual training \yli{scheme}.
Instead of \yli{minimising} the loss 
$\mathcal{L}\left(F_{GT},\widehat{F}\right)$, 
where $\widehat{F}$ is the predicted flow variable, $F_{GT}$ is the corresponding ground truth, and $\mathcal{L}$ is an arbitrary loss criterion, 
the network minimises the residual loss $\mathcal{L}\left(F_{GT},\widehat{F}_r + Upsample(F_{LR})\right)$. Here, $\widehat{F}_r$ is the predicted residual and the prediction is $\widehat{F} = \widehat{F}_r + Upsample(F_{LR})$. Since the low-resolution field is an approximation of the ground truth, much of the residual field will be close to zero. Thus, training the model to predict the residual field eases the learning, and helps the model to focus on \yli{the} more nuanced areas where the low-resolution fields tend to be inaccurate.

\jess{
A similar method, referred to as "learned correction", was demonstrated by \citet{kochkov2021}. However, it depends on the introduction of the low-resolution field as a new input to an the neural network in an originally non-super-resolution case. In the next section, this work intends to support the addition of the low-resolution field to the network output as a key step in improving an already super-resolution case.
}

%% file: 4_expt.tex
\section{Experiments}
\label{sec:expt}

We evaluate our proposed geometric and finite volume features on two databases, 2DSHAPES~\citep{viquerat2020,chen2022} and AirfRANS~\citep{airfrans}, and
\jess{five} state-of-the-art GNN methods and their respective training schemes in the CFD--AI literature, \jess{MeshGraphNet~\cite{deepmind2020}, BSMS-GCNN~\cite{bsmsgnn2022},} Chen-GCNN~\citep{chen2021}, Graph U-Net~\citep{airfrans}, and  CFDGCN~\citep{cfdgcn}. While all state-of-the-art methods utilise mesh node-based graphs (Figure~\ref{fig:meshgraph}), 
we adapt both datasets to cell centroid-based graphs (Figure~\ref{fig:cfdgraph}) to train the models with FVF. Models that do not use FVF are trained on mesh node-based graphs to maintain consistency with the baselines. The effectiveness of residual training is demonstrated only on CFDGCN with the AirfRANS dataset because CFDGCN is the only method that takes in low-resolution data.

\textbf{2DSHAPES} is a collection of 2000 random shapes where each shape is generated by connected \naheed{B\'ezier curves} along with their steady-state velocity 
vector $[\bar{u}_x,\bar{u}_y]^T$ and pressure $\bar{p}$, 
at Reynolds number = 10. The training set, validation set, and test set consist of 1600, 200, and 200 shapes, respectively. We use this dataset to demonstrate the effectiveness of our features applied to the Chen-GCNN model. 


\textbf{AirfRANS} is a high fidelity aerodynamics dataset of airfoil shapes. Their (RANS) solutions range from Reynolds numbers 2 to 6 million and angle-of-attacks $-5^\circ$ to $+15^\circ$ degrees. In their scarce data regime, the training set, validation set, and test set consist of respectively 180, 20, and 200 airfoils along with their steady-state velocity %
vector $[\bar{u}_x,\bar{u}_y]^T$, pressure $\bar{p}$, and turbulent viscosity $\nu_t$.
This dataset is used as-is to demonstrate the effectiveness of our features applied to the Graph U-Net. An amended dataset, which excludes the turbulent viscosity field, is likewise used for comparing \jess{BSMS-GNN and }CFDGCN with and without the proposed features. \jess{Similarly, the dataset is amended to further exclude the pressure fields for the MeshGraphNet models. 
These fields are excluded to more closely match those predicted by the original methods. }Additionally, we use the authors’ codebase~\citep{airfrans_codebase}
to generate RANS solutions of the same set of airfoils on coarser meshes ($\sim20$ times fewer nodes than the original), which we refer to as \textit{coarse-AirfRANS} and utilise as a database of low-resolution reference flow fields to evaluate the effectiveness of residual training on CFDGCN.

\paragraph{Performance metrics.}Our objective is to determine if incorporating the proposed features \naheed{and residual training into the state-of-the-art methods significantly improves performance.} We evaluate different methods with different measures to match the metrics in their respective studies.
\naheed{To compare Chen-GCNN with our models derived from it, we use the MAE loss function, which is the summation of mean-absolute error of flow variables averaged over the test set. For Graph U-Net from Bonnet et al.~\citep{airfrans}, the following \jess{three} types of metrics are employed:}
\eat{
\begin{enumerate}
    \item \textit{Test Loss:} Loss function (Equation 3 in~\citep{airfrans}) evaluated on and averaged over the test set.
	\item \textit{Volume MSE:} Mean-squared errors $\bar{u}_x, \bar{u}_y, \bar{p},\nu_t$ of normalised flow field prediction at internal nodes.
	\item \textit{Surface MSE:} Mean-squared error of normalised flow field prediction at boundary nodes. Due to boundary constraints, only pressure $(\bar{p})$  is non-zero at airfoil surfaces.
\end{enumerate}
}
\begin{enumerate}
\item \textit{Test loss:} Loss function by equation~3 in~\citet{airfrans} evaluated on and averaged over the test set.
\item \textit{Volume MSE:} \naheed{Mean-squared errors of normalised flow field \yli{($[\bar{u}_x,\bar{u}_y]^T$, $\bar{p}$, $\nu_t$)} predictions at internal nodes.}
\item \textit{Surface MSE:} Mean-squared error of normalised flow field predictions at boundary nodes. Due to boundary constraints, only pressure $(\bar{p})$  is non-zero at airfoil surfaces.
\end{enumerate}
Finally, we evaluate \jess{the models derived from MeshGraphNet and CFDGCN} using \yli{the} MSE loss function, \yli{which is} the mean-squared error of flow field prediction averaged over the test set\jess{, and the models derived from BSMS-GNN with the RMSE or root mean-squared error}.

\subsection{Results}
\paragraph{MeshGraphNet~\cite{deepmind2020}.} \jess{The MeshGraphNet model uses an encoder-processor-decoder architecture. The original convolution types were maintained in the experiments. Table~\ref{tab:MGN_results} shows that the adoption of the geometric features (Geo) and the Finite Volume Features (FVF) reduced the MSE by $\sim$41\% from the baseline.}

\begin{table}[htb!]
\centering
\begin{tabular}{c|c}
\toprule 
\textbf{Models} & \textbf{MSE} $\times10^{-2}$ \\ \midrule 
MeshGraphNet (Baseline) & 5.7571 \\
MeshGraphNet w/ Geo & 3.4683 \\
MeshGraphNet w/ FVF w/ Geo & \textbf{3.3811} \\
\bottomrule 
\end{tabular}
\caption{Performance evaluation using MeshGraphNet on the AirfRANS dataset.}
\label{tab:MGN_results}
\end{table}\vspace{-0.5em}

\paragraph{BSMS-GNN~\cite{bsmsgnn2022}.} \jess{The Bi-Stride Multi-Scale GNN (BSMS-GNN) model also has an encoder-processor-decoder architecture, similar to the MeshGraphNet. Likewise, only the original convolution types were used in the experiment.
The BSMS-GNN is a multi-scale model, and hence the FVF implementations were only applied to the encoder, decoder, and first and last convolutions of the processor, where the graph is at full resolution. As can be seen in Table~\ref{tab:BSMS_results}, the adoption of our methods results in a $\sim$20\% reduction in the RMSE.}

\begin{table}[htb!]
\centering
\begin{tabular}{c|c}
\toprule 
\textbf{Models} & \textbf{RMSE} $\times10^{-2}$ \\ \midrule 
BSMS-GNN (Baseline) & 7.7589 \\
BSMS-GNN w/ Geo & 6.7498 \\
BSMS-GNN w/ FVF w/ Geo & \textbf{6.1919} \\
\bottomrule 
\end{tabular}
\caption{Performance evaluation using BSMS-GNN on the AirfRANS dataset.}
\label{tab:BSMS_results}
\end{table}\vspace{-1em}

\paragraph{Chen-GCNN~\citep{chen2021}.}
On the Chen-GCNN model using invariant edge convolutions, \naheed{Table~\ref{tab:chengcnn} shows that} the \naheed{adoption} of the geometric features and finite volume features
\naheed{reduces MAE} by $\sim$27\% from the baseline.
We also tested GCNN with SAGE convolution and obtained $\sim$82\% reduction in MAE with respect to the baseline.


\begin{table}[htb!]
\centering
 \begin{tabular}{@{}c|c|c@{}}
\toprule
\multirow{2}{*}{\textbf{Models}} &
  \multirow{2}{*}{\textbf{Conv type}} &
  \multirow{2}{*}{\textbf{\begin{tabular}[c]{@{}c@{}}MAE \\ $\times 10^{-2}$\end{tabular}}} \\
   &  & \\ \midrule
Chen-GCNN (Baseline) &
  \multirow{3}{*}{\begin{tabular}[c]{@{}c@{}}Invariant\\ Edge\\ Convolution\end{tabular}} &
  1.1590 \\
Chen-GCNN w/ Geo        &  & 0.9727          \\
Chen-GCNN w/ FVF w/ Geo &  & \jess{\textbf{0.8491}} \\ \midrule \midrule
Chen-GCNN (Baseline) &
  \multirow{3}{*}{\begin{tabular}[c]{@{}c@{}}SAGE\\ Convolution\end{tabular}} & 6.7103 \\
Chen-GCNN w/ Geo        &  & 3.8041          \\
Chen-GCNN w/ FVF w/ Geo &  & \textbf{1.1982} \\ \bottomrule
\end{tabular}%
\caption{Performance evaluation using Chen-GCNN on the 2DSHAPES dataset.}
\label{tab:chengcnn}
\end{table}\vspace{-1em}

\paragraph{Graph U-Net~\citep{airfrans}.}
We trained Graph U-Net architecture (baseline) with the same experimental set-up described in Appendix L 
of \naheed{Bonnet et al.~\citep{airfrans}}. 
The reason for choosing Graph U-Net is that it is the best-performing model reported. We evaluated this baseline with two convolution types: SAGE convolution and vanilla graph convolution. 
In Table \ref{tab:gunet}, we observe that, on SAGE convolution, the \naheed{adoption of the geometric and finite volume features reduces the loss on \naheed{the} test set by about 21\% as indicated by the baseline method's loss on the test set $(1.816\times 10^{-2})$ and that of Graph U-Net w/ FVF w/ Geo $(1.441\times 10^{-2})$}. We also observe that our features improve the volume MSE of \yli{three} out of \yli{four} flow variables as well as \naheed{the surface MSE} of pressure.
We also tested Graph U-Net \naheed{with vanilla graph convolution} and obtained $71\%$ reduction \naheed{in the loss on the test set with respect to the baseline}.
\begin{table*}[t]
\centering
\begin{tabular}{@{}c|c|c|cccc|c@{}}
\toprule
\multirow{2}{*}{\textbf{Models}} &
  \multirow{2}{*}{\textbf{\begin{tabular}[c]{@{}c@{}}Conv\\ type\end{tabular}}} &
  \multirow{2}{*}{\textbf{\begin{tabular}[c]{@{}c@{}}Test \\ loss\\ $\times 10^{-2}$\end{tabular}}} &
  \multicolumn{4}{c|}{\textbf{\begin{tabular}[c]{@{}c@{}}Volume MSE\\ $\times 10^{-2}$\end{tabular}}} &
  \textbf{\begin{tabular}[c]{@{}c@{}}Surface\\ MSE\\ $\times 10^{-1}$\end{tabular}} \\ 
 &
   &
   &
  $\bar{u}_x$ &
  $\bar{u}_y$ &
  $\bar{p}$ &
  $\nu_t$ &
  $\bar{p}$ \\ \midrule
Graph U-Net (Baseline) &
  \multirow{3}{*}{\begin{tabular}[c]{@{}c@{}}SAGE\\ Convolution\end{tabular}} &
  1.816 &
  1.140 &
  1.429 &
  2.190 &
  \textbf{2.492} &
  0.932 \\
Graph U-Net w/ Geo &
   &
  1.786 &
  1.224 &
  0.599 &
  2.028 &
  3.286 &
  0.828 \\
Graph U-Net w/ FVF w/ Geo &
   &
  \textbf{1.441} &
  \textbf{1.061} &
  \textbf{0.515} &
  \textbf{1.368} &
  2.724 &
  \textbf{0.549} \\ \midrule \midrule
Graph U-Net (Baseline) &
  \multirow{3}{*}{\begin{tabular}[c]{@{}c@{}}Vanilla\\ Graph\\ Convolution\end{tabular}} &
  15.310 &
  19.249 &
  12.845 &
  14.041 &
  15.190 &
  4.633 \\
Graph U-Net w/ Geo &
   &
  12.709 &
  17.408 &
  11.256 &
  10.106 &
  12.146 &
  3.631 \\
Graph U-Net w/ FVF w/ Geo &
   &
  \textbf{4.544} &
  \textbf{2.968} &
  \textbf{2.286} &
  \textbf{4.546} &
  \textbf{8.287} &
  \textbf{1.505} \\ \bottomrule
\end{tabular}%
\caption{Performance evaluation using Graph U-Net on the AirfRANS dataset.}
\label{tab:gunet}
\end{table*}
\begin{table*}[t]
\centering
 \begin{tabular}{@{}c|c|c@{}}
\toprule
\multirow{2}{*}{\textbf{Models}} &
  \multirow{2}{*}{\textbf{Conv type}} &
  \multirow{2}{*}{\textbf{\begin{tabular}[c]{@{}c@{}}MSE \\ $\times 10^{-2}$\end{tabular}}} \\
                        &  &                 \\ \midrule
CFDGCNN (Baseline) &
  \multirow{4}{*}{\begin{tabular}[c]{@{}c@{}}Vanilla\\ Graph\\ Convolution\end{tabular}} &
  0.1211 \\
CFDGCN w/ Geo        &  & 0.1093          \\
CFDGCN w/ FVF w/ Geo &  & 0.0918 \\ 
CFDGCN w/ residual training w/ FVF w/ Geo &  & \textbf{0.0719} \\\midrule \midrule
CFDGCN (Baseline) &
  \multirow{4}{*}{\begin{tabular}[c]{@{}c@{}}SAGE\\ Convolution\end{tabular}} &
  0.1342 \\
CFDGCNN w/ Geo        &  & 0.1092 \\
CFDGCNN w/ FVF w/ Geo &  & 0.0631 \\ 
CFDGCNN w/ residual training w/ FVF w/ Geo &  & \textbf{0.0628} \\ \bottomrule
\end{tabular}%
\caption{Performance evaluation using CFDGCN on the AirfRANS dataset.}
\label{tab:cfdgcn}
\end{table*}
\paragraph{CFDGCN~\citep{cfdgcn}.}
CFDGCN uses a CFD simulator in \naheed{the training and testing scheme} to generate low-resolution data and vanilla graph convolution to enhance learning on a high-resolution mesh with varying initial physical conditions, e.g., angle of attack and Mach number. \naheed{Although CFDGCN generalises well to unknown physical conditions, it} requires the geometry to remain fixed for the purpose of optimising the coarse mesh. This makes the model unable to generalise to different geometries and limits the method's practical functionality.
Hence, in our adaption, we removed their differentiable CFD simulator component \naheed{from the training and testing loop and} directly passed the coarse-AirfRANS solutions for upsampling. As the experimental set-up assumed that a low-resolution flow field is available, we adopted and evaluated the \naheed{proposed residual training scheme} in addition to Geo and FVF. In Table~\ref{tab:cfdgcn}, we observe that the predictive error reduces by about 41\%, \naheed{as indicated by the baseline method's MSE $(0.1211 \times 10^{-2})$ 
and that of CFDGCN w/ residual training w/ FVF w/ Geo $(0.0719 \times 10^{-2})$.} 
Incorporating our proposed features and residual training scheme one by one consistently reduces the error, 
suggesting that they are generally effective. 
\naheed{We also tested CFDGCN with SAGE convolution and obtained a 53\% reduction in MSE with respect to the baseline}.
\subsection{\jess{Computation Time of the Proposed Features}} 
\jess{
We present the 
running time of computing the proposed features on Coarse-AirfRANS and (Fine) AirfRANS datasets, shown in Table~\ref{tab:compute_time}. 
The average computation times of SV and FVF on coarse and fine meshes are quite reasonable. 
The reason the computational time of DID is relatively higher than that of other features is 
because the values of each angle segment are calculated consecutively rather than in parallel.
Computational time can be reduced by calculating them in parallel instead, or by using fewer angle segments.
}
\begin{table}[htb!]
\centering
\begin{tabular}{c|cc}
\toprule
 & \multicolumn{2}{c}{\textbf{Average Computation Time (s)}} \\ \cmidrule{2-3} 
 & \multicolumn{1}{c|}{Coarse-AirfRANS} & AirfRANS \\ \midrule
SV & \multicolumn{1}{c|}{0.001} & 0.018\\
DID & \multicolumn{1}{c|}{0.031} & 4.276 \\
FVF & \multicolumn{1}{c|}{0.001} & 0.032 \\ \bottomrule
\end{tabular}
\caption{Computational time of SV, DID and FVF for two mesh resolutions. 
}
\label{tab:compute_time}
\end{table}

%% file: 5_conc.tex
\section{Conclusion}
\label{sec:conc}
This work presents two novel geometric representations, SV and DID, and the use of FVF in graph convolutions. The SV and DID provide a more complete representation of the geometry to each node. Moreover, the FVF enable the graph convolutions to more closely model the finite volume simulation method. Their effectiveness at reducing prediction error \naheed{has been} shown across two datasets, as well as \jess{five} different state-of-the-art methods and training scenarios 
using various types of graph convolution. Additionally, this paper demonstrates the ability of residual training to further improve accuracy in scenarios with low-resolution data. \vspace{-0.5em}
\section*{Acknowledgement}
The research, undertaken in the Rolls-Royce Corporate Lab
@ Nanyang Technological University, is supported by the
Singapore Government (Industry Alignment fund IAF-ICP
Grant - I1801E0033). We thank Wang Yi and Bryce Conduit from Rolls-Royce plc. for their feedback on this work.

%% file: 6_appendix.tex
\appendix
\section*{Technical Appendix}
\section{Relation Between a CFD Mesh and its FVF}
\label{sec:FVFproof}
In this section, we show that any common 2D mesh can be reconstructed
from the FVF. Specifically, assuming that in a 2D mesh, 
\begin{enumerate}
    \item the cells along the farfield have no more than two boundary faces each, and
    \item the faces of each cell enclose a singular volume,
\end{enumerate}
the relative positions of all the nodes and the faces of a mesh can be uniquely deduced given
\begin{enumerate}
    \item the volumes of all cells $\textrm{V}_i$,
    \item the face area normal vector $\textbf{S}_{f,ij}$ of all internal and geometry boundary faces, and
    \item \jess{the vector from the  cell centroids to the face centroids}
    $(\bm{c}_{f,ij}-\bm{x}_{i})$ and $(\bm{c}_{f,ij} - \bm{x}_j)$ of all internal and geometry boundary faces.
\end{enumerate}
By the "relative positions" $\bm{a}'_i$ of mesh nodes that have the absolute positions $\bm{a}_i$, we mean that there exists some constant $\alpha$ such that $\bm{a}_i=(\bm{a}'_i-\alpha), \forall i$.
By the principle of translational invariance, the absolute positions of the mesh points are inconsequential to the flowfield as long as their relative positions remain the same. We make no further notational distinction between the two in this section.

\begin{lemma}
    Given the (relative) position of the face centroid $\bm{c}_{f,ij}$ and face area normal vector $\textbf{S}_{f,ij} = (S_x,S_y)$ of a face in a 2D mesh, the unique (relative) positions of the mesh nodes $\bm{a}, \bm{b}$ of the face, and the existence of the face, can be deduced.
    \label{lemma1}
    \begin{proof}
    From the (relative) position of the face centroid $\bm{c}_{f,ij}$, we know that $(\bm{a}+\bm{b})/2=\bm{c}_{f,ij}$. Likewise, we can assume without loss of generality that $(\bm{b}-\bm{a})$ is the face area normal vector $\textbf{S}_{f,ij}$ rotated by a right angle clockwise, or $(\bm{b}-\bm{a}) = (S_y,-S_x)$. The unique (relative) positions of the mesh nodes $\bm{a}, \bm{b}$ can be solved from this system of two linear equations. The existence of a face shared by them is obvious.
    \end{proof}
\end{lemma}
\begin{lemma}
    Given the (relative) position of a cell's centroid $\bm{x}_{i}$ and the (relative) positions of all but one of its mesh nodes  $\bm{b}_1,\dots,\bm{b}_{n-1}$, the unknown (relative) position of its remaining mesh node $\bm{a}$ can be found to be $n\bm{x}_i-\sum_{j=1}^{n-1} \bm{b}_j$.
    \label{lemma2}
    \begin{proof}
    From the (relative) position of the cell centroid $\bm{x}_{i}$, we know that $\frac{1}{n}\left(\bm{a}+\sum_{j=1}^{n-1} \bm{b}_j\right)=\bm{x}_{i}$. The rest is obvious.
    \end{proof}
\end{lemma}
\begin{lemma}
    If and only if $\bm{x}_{i}$ is the (relative) position of a cell centroid with the mesh nodes of (relative) positions
    $\bm{b}_1,\dots,\bm{b}_{n-1}$, then $\bm{x}_i=n\bm{x}_i-\sum_{j=1}^{n-1} \bm{b}_j$.
    \label{lemma3}
    \begin{proof}
    \[
    \begin{aligned}
        &\frac{1}{n-1}\sum_{j=1}^{n-1} \bm{b}_j=\bm{x}_i\Leftrightarrow\sum_{j=1}^{n-1} \bm{b}_j=(n-1)\bm{x}_i\quad\\
        &\quad\Leftrightarrow\sum_{j=1}^{n-1} \bm{b}_j=n\bm{x}_i-\bm{x}_i\Leftrightarrow \bm{x}_i=n\bm{x}_i-\sum_{j=1}^{n-1} \bm{b}_j
    \end{aligned}
    \]
    \end{proof}
\end{lemma}
\begin{theorem} [The Completeness of FVF]
    Let $M$ be a 2D mesh such that the cells along its farfield have no more than $2$ boundary faces each, and the faces of each of its cells enclose a singular volume.
    If a graph $G=(V,E)$ is the cell centroid-based graph representation of $M$, the relative positions of all the nodes and faces of $M$ can be uniquely deduced given {\normalfont $\textrm{V}_i$, } $\forall i\in V$, and $\textbf{S}_{f,ij}$, $(\bm{c}_{f,ij}-\bm{x}_{i})$ and $(\bm{c}_{f,ij} - \bm{x}_j)$, $\forall(i,j)\in E$.
\end{theorem}
\begin{proof}
By the cell-centroid based graph construction method, each cell of the input mesh corresponds to a node $i\in V$, and each internal face and geometry boundary face of the input mesh corresponds to an edge $(i, j) \in E$.

The vectors $(\bm{c}_{f,ij}-\bm{x}_{i})$ and $(\bm{c}_{f,ij} - \bm{x}_j)$ are given in the edge attributes of the FVF. It is obvious that the positions of all face centroids $\bm{c}_{f,ij}$ and cell centres $\bm{x}_{i}, \bm{x}_j$ relative to one another in a connected mesh can be deduced from this. The face area vector $\textbf{S}_{f,ij}$ is also included as an edge attribute. Hence, from Lemma \ref{lemma1}, all internal faces, geometry adjacent boundary faces, and their respective mesh nodes can be found, shown in Figure \ref{fig:theorem1a}. This already accounts for the significant majority of the original mesh. All that remains are the edges along the farfield boundary, as well as the relative position of any mesh nodes that did not abut an internal face (such as those at corners).

As mentioned before, the relative positions of the cell centroids $\bm{x}_i$ can be deduced by the edge attributes. Assuming that these cells do not have more than $2$ boundary faces along the farfield each, they will have at most $1$ mesh node not associated with an internal edge. This is a fair assumption, considering the typical meshes with triangular or quadrilateral cells, commonly used for 2D CFD simulations.
If there is a node with an unknown relative position, by Lemma \ref{lemma2}, it can be found as $n\bm{x}_i-\sum_{j=1}^{n-1} \bm{b}_j$ where $b_1,\dots,b_{n-1}$ are the known mesh positions. Alternatively, if $n\bm{x}_i-\sum_{j=1}^{n-1} \bm{b}_j$ is equal to the cell centroid as in Lemma \ref{lemma3}, it can be concluded that $\{b_1,\dots,b_{n-1}\}$ is the complete set of mesh nodes. 
Finally, the unique faces connecting these mesh nodes along the farfield can be deduced from the assumption that the faces of each cell along the farfield should enclose a singular volume. This is illustrated in Figure \ref{fig:theorem1b}.
\end{proof}

\begin{figure}[tb!]
     \centering
     \begin{subfigure}[b]{0.45\textwidth}
         \centering
         \includegraphics[width=\textwidth]{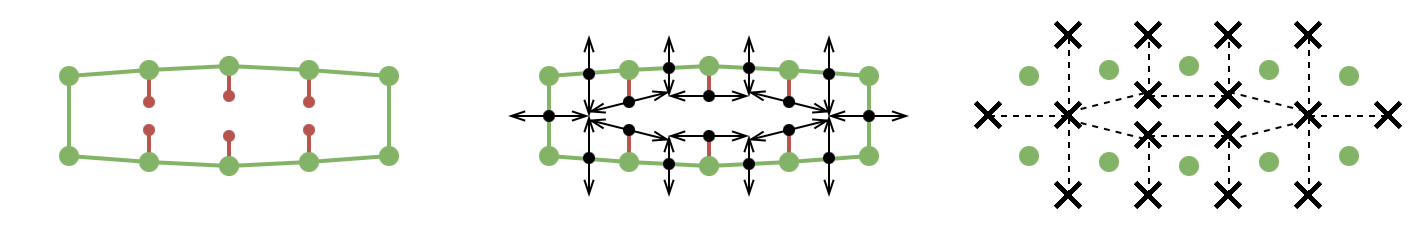}
         \caption{Obtaining the mesh nodes and faces of all internal and geometry boundary faces.}
         \label{fig:theorem1a}
     \end{subfigure}
     \hfill
     \begin{subfigure}[b]{0.45\textwidth}
         \centering
         \includegraphics[width=\textwidth]{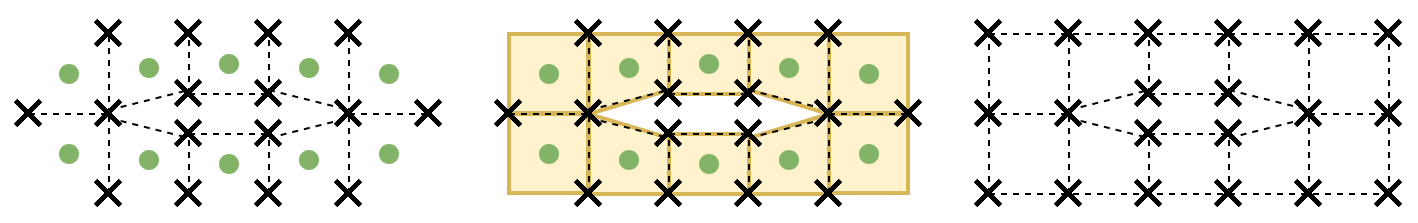}
         \caption{Obtaining corner mesh nodes and farfield boundary faces.}
         \label{fig:theorem1b}
     \end{subfigure}
        \caption{Reconstructing the original CFD mesh from the FVF.}
        \label{fig:theorem1}
\end{figure}


\section{Implementation Details}

\subsection{Numerical DID Implementation}
While the continuous DID has already been defined, the geometry boundary had to be represented as a parametric equation, which may not always be available. Instead, we calculate a discrete DID using numerical integration
\jess{, described in detail in Algorithm \ref{alg:DID}. Alternatively, the numerical DID calculation of an angle range can be simplified as:
\begin{enumerate}
    \item Obtain the mean distance to the node from every unobstructed boundary point: $d$.
    \item Find the weighted proportion of the angle segment which faces the object boundary: $w_\theta$.
    \item The DID value is calculated to be the weighted average of the mean distance to the boundary and the $\infty$ parameter: $w_\theta\cdot d+(1-w_\theta)\cdot\infty$.
\end{enumerate}

 It can be seen from Algorithm \ref{alg:DID} that the process runs in $\mathcal{O}(|K|*|V|)$ time, where $K$ is the set of boundary nodes and $V$ is the set of all nodes. However, as $|K|<<|V|$ in most practical CFD scenarios, the computation cost is almost linear to $|V|$. Also note that a singular geometry was assumed. For multiple-geometry scenarios, it has to be considered that $w_\theta$ may not be represented by a single continuous angle range like $(\theta_{min},\theta_{max})$. The rest of the algorithm may remain unchanged.
}

In all experiments, the starting and ending angles \jess{of the DID segments} were 
$\left(\theta_j,\theta'_j\right)=\left(-\frac{\pi}{4},\frac{\pi}{4}\right)$,
$\left(0,\frac{\pi}{2}\right)$,
$ \left(\frac{\pi}{4},\frac{3\pi}{4}\right)$,
$\left(\frac{\pi}{2},\pi\right)$,
$\left(\frac{3\pi}{4},\frac{5\pi}{4}\right)$,
$\left(\pi,\frac{3\pi}{2}\right)$,
$\left(\frac{5\pi}{4},\frac{7\pi}{4}\right)$,
$\left(\frac{3\pi}{2},2\pi\right)$.
These were chosen to give 8 overlapping arcs each spanning $\frac{\pi}{2}$ degrees, centered at $\frac{\pi}{4}$ intervals. The respective weight functions $w_j(\theta)$ were uniform distributions over the $\left(\theta_j,\theta'_j\right)$ range. Finally, $4$ was used as the large number $\infty$, as this was larger than any finite distance from a point in the mesh to the geometry boundary in any direction.

\begin{algorithm*}[htb!]
\caption{Calculation of the DID of a field}
\label{alg:DID}
\begin{algorithmic}[1]
\Require Set of nodes $V$, positions of each node $pos=[(x_i,y_i):i\in V]$ and boundary indices $bd=[k\in V:k\textrm{ is on the geometry boundary}]$, starting and ending angles $\left[\left(\theta_j,\theta'_j\right):0\leq j<J\right]$, weight functions $[w_j:0\leq j<J]$, and large value $\infty$.
\Ensure $DID$ values $DID_j$ for every segment, for all $i\in V$.
\State Initialise $DID\gets[~]$\\
\Comment{Compute $DID$ for every segment $j$}
\For{$j\in[0,...,J-1]$}
    \State Initialise $DID_j\gets[~]$\\
    \Comment{Compute $DID_i$ for every node $i\in V$}
    \For{$i\in V$}
        \State Initialise $\theta_{min},\theta_{max}\gets\theta'_j,\theta_j$
        \State Initialise $DID_i\gets 0$
        \State Initialise $w\gets 0$\\
        \Comment{Compute distance from point $i$ to every boundary point within its segment range}
        \For{$k\in bd$}
            \State Initialise $\theta_{i,k}\gets\tanh\left((y_j-y_i)/(x_k-x_i)\right)$
           \If{$(\theta_j<\theta_{i,k}<\theta'_j)$ and $(k\textrm{ is  \jess{unobstructed from } }i)$}
                \State $DID_i\gets DID_i+w_j(\theta_{i,k})*\min(||(x_i,y_i)-(x_k,y_k)||,\infty)$\\
                \Comment{Update sum of discrete weights}
                \State $w\gets w+w_j(\theta_{i,k})$\\
                \Comment{Update minimum angle range from point $i$ to point $k$}
                \State $\theta_{min}\gets \theta_{i,k}$ if $\theta_{i,k}< \theta_{min}$
                \State $\theta_{max}\gets \theta_{i,k}$ if $\theta_{i,k}> \theta_{max}$
            \EndIf
        \EndFor\\
        \Comment{Compute weight of mininum angle segment \jess{that faces a boundary}}
        \State $w_{\theta}\gets\left(\int_{\theta_{min}}^{\theta_{max}}w_j(\theta)\,d\theta\right)/\left(\int_{\theta_j}^{\theta'_j}w_j(\theta)\,d\theta\right)$
        \State $DID_i\gets DID_i/w$
        \State $DID_i\gets w_{\theta}*DID_i+(1-w_{\theta})*\infty$
    \EndFor
    \State $DID_j.append(DID_i)$
\EndFor
\State $DID.append(DID_j)$
\end{algorithmic}
\end{algorithm*}

\subsection{FVF implementation}

\paragraph{Vanilla Graph Convolution and SAGE Convolutions}
For models using vanilla graph convolutions, the convolutions were entirely replaced by the FVGC 
earlier described when FVF was used.
For models using SAGE convolutions, FVF was used by implementing the FVGC as the aggregation function of the SAGE convolution.
In both cases, the hidden dimension, or the number of filters, used in each convolution was always $3$ for all convolutions in all experiments.

\paragraph{Invariant Edge Convolutions}
Following the notation of the original paper, the Invariant Edge (IVE) convolutions used by \citet{chen2021} update the edge features $\bm{x}_e$ and node features $\bm{x}_v$ into $\bm{x}'_e$ and $\bm{x}'_v$ respectively, according to the following rules
\begin{align}
    \bm{x}'_e&=f_e\left(\frac{\bm{x}_{v_1}+\bm{x}_{v_2}}{2}, \frac{|\bm{x}_{v_1}-\bm{x}_{v_2}|}{2}, \bm{x}_e\right)\;,\\
    \bm{x}'_v&=f_v\left(\bm{x}_v, \sum_{e_i\in N(v)}\bm{x}'_{e_i}\right)\;,
\end{align}
where $v_1$ and $v_2$ are the two nodes connected by edge $e$, $N(v)$ is the set of neighbouring edges around node $v$, and $f_e$ and $f_v$ are both MLPs with 1 hidden layer of 128 neurons. The number of neurons in the output layer, and hence output dimension of the kernel, can be customised. This is similar to the use of multiple filters in FVGC.
Hence, when using FVF with IVE convolutions, only the node and edge attributes have to be implemented, resulting in the following rules
\jess{
\begin{align}
    &\bm{x}^*_e=\bm{q}_e\oplus\bm{x}_e\;,\quad\bm{x}^*_{v}=\bm{x}_{v}\oplus\bm{p}\;,\\
    &\bm{x}'_e=f_e\left(\frac{\bm{x}^*_{v_1}+\bm{x}^*_{v_2}}{2}, \frac{|\bm{x}^*_{v_1}-\bm{x}^*_{v_2}|}{2}, \bm{x}^*_e\right)\;,\\
    &\bm{x}'_v=f_v\left(\bm{x}^*_v, \sum_{e_i\in N(v)}\bm{x}'_{e_i}\right)\;,
\end{align}
}
\jess{where $\bm{p}$ is the node attributes and $\bm{q_e}$ is the edge attributes of the edge $e$.}
Note that the new edge and node features are still $\bm{x}'_e$ and $\bm{x}'_v$, while $\bm{x}^*_e$ and $\bm{x}^*_v$ are just intermediate calculations.

Unlike in vanilla graph convolutions or SAGE convolutions, the hidden dimension is not fixed at $3$, but will instead follow the number of intermediate edge features used in the original work as described in the next section.


\paragraph{MeshGraphNet Convolution Implementation}
\jess{
Still following the notation of the original paper, the MeshGraphNet processor designed by \citet{deepmind2020} used convolutions that updated the edge features $\bm{e}^M$ and node features $\bm{v}$ into $\bm{e}'^M$ and $\bm{v}'$ according to the following
\begin{equation}
\begin{aligned}
    &\bm{e}'^M_{ij}=f^M(\bm{e}^M_{ij},\bm{v}_i,\bm{v}_j)\;,\\
    &\bm{v}'_i=f^V(\bm{v}_i,\sum_j\bm{e}'^M_{ij})\;.
\end{aligned}
\end{equation}
Where $v_i$ and $v_j$ are the features of two nodes connected by an edge with edge features $e^M_{ij}$, and $f^M$ and $f^V$ are MLPs with 2 hidden layers of size 128 each. In all MeshGraphNet implementations, they have an output of size 128 as well.
When using FVF, it was changed to the following
\begin{equation}
\begin{aligned}
    &\bm{e}^{*M}_{ij}=\bm{e}^M_{ij}\oplus \bm{q}_{ij},\quad \bm{v}^*_i=\bm{v}_i\oplus \bm{p}_i\;,\\
    &\bm{e}'^M_{ij}=f^M(\bm{e}^{*M}_{ij},\bm{v}^*_i,\bm{v}^*_j)\;,\\
    &\bm{v}'_i=f^V(\bm{v}^*_i,\sum_j\bm{e}'^M_{ij})\;.
\end{aligned}
\end{equation}
We also appended the FV node attributes and FV edge attributes to the inputs of the encoder and decoder convolutions accordingly.
}

\jess{
\paragraph{BSMS-GNN Convolution Implementation}
Likewise, the original BSMS-GNN processor designed by \citet{bsmsgnn2022} used convolutions to update the edge features $\bm{e}_l^s$ and node features $\bm{v}_l$ at level $l$ for a problem involving $S$ edge sets like so
\begin{equation}
\begin{aligned}
    &\bm{e}^s_{l,ij}=f^s_l(\Delta \bm{x}_{l,ij},\bm{v}_{l,i},\bm{v}_{l,j}),\quad s=1,\dots,S\;,\\
    &\bm{v}'_{l,i}=\bm{v}_{l,i}+f^V_l(\bm{v}_{l,i},\sum_j\bm{e}^1_{l,ij},\dots,\sum_j\bm{e}^S_{l,ij})\;.
\end{aligned}
\end{equation}
Where $\Delta \bm{x}_{l,ij}=\bm{x}_i-\bm{x}_j$ is the relative positions of the nodes $i$ and $j$, and $f_l^s$ and $f_l^V$ are MLPs with 2 hidden layers of a hidden dimension of 128. For all BSMS-GNN implementations, an output of size 128 was used as well.
With the FVF, it was changed into
\begin{equation}
\begin{aligned}
    &\bm{v}^*_{l,i}=\bm{v}_{l,i}\oplus p_i\quad,\\
    &\bm{e}^s_{l,ij}=f^s_l(\Delta x_{l,ij},\bm{v}^*_{l,i},\bm{v}^*_{l,j},q_{ij}),\quad s=1,\dots,S,\\
    &\bm{v}'_{l,i}=\bm{v}^*_{l,i}+f^V_l(\bm{v}^*_{l,i},\sum_j\bm{e}^1_{l,ij},\dots,\sum_j\bm{e}^S_{l,ij})
\end{aligned}
\end{equation}
As mentioned before, this was only done for the first and last convolutions of the processor where the graph is at full resolution. Similar to the MeshGraphNet convolution implementation, the FV node attributes were appended to the inputs of the encoder and decoder convolutions as well.
}

\subsection{Model Architectures and Hyperparameter Choices}
\begin{table*}[t!]
\centering
\begin{tabular}{c|c}
\toprule 
\textbf{Models} & \textbf{Number of trainable parameters} \\ \midrule 
MeshGraphNet (Baseline) & 2282370 \\
MeshGraphNet w/ Geo & 2283650 \\
MeshGraphNet w/ FVF w/ Geo & 2302470 \\
MeshGraphNet w/ Residual Training w/ FVF w/ Geo & 2302470 \\
\bottomrule 
\end{tabular}
\caption{Sizes of the MeshGraphNet models compared.}
\label{tab:MGN_params}
\end{table*}
\begin{table*}[htb!]
\centering
\begin{tabular}{c|c}
\toprule 
\textbf{Models} & \textbf{Number of trainable parameters} \\ \midrule 
BSMS-GNN (Baseline) & 2578947 \\
BSMS-GNN w/ Geo & 2580227 \\
BSMS-GNN w/ FVF w/ Geo & 2582787 \\
BSMS-GNN w/ Residual Training w/ FVF w/ Geo & 2582787 \\
\bottomrule 
\end{tabular}
\caption{Sizes of the BSMS-GNN models compared.}
\label{tab:BSMS_params}
\end{table*}

\paragraph{MeshGraphNet}
\jess{
The MeshGraphNet used an encoder-processor-decoder structure. The encoder consisted of two MLPs, for encoding node and edge features respectively. The input node and edge dimensions were both 3 for the baseline model. The number of input node channels increased to 13 if SV and DID were being used. Likewise, the input node channels increased again to 14, and the input edge features to 9, if FVF was implemented. It had a hidden and output size of 128. The decoder consisted of just one MLP to decode the node features only, with an input and hidden size of 128. If FVF was being used, the input size increased to 129. The ouput size was 2, for the velocity fields.

The processor had 15 layers of the MeshGraphNet convolutions as earlier described. The input node and edge sizes were 128 unless FVF was implemented when the number of node channels increased to 129, and likewise, the edge channels increased to 134. The number of output nodes and edge channels was consistently 128. All activations used in the model were ReLU. The resulting number of trainable parameters in each model is shown in Table~\ref{tab:MGN_params}.

Each model was trained for 250 epochs on half-precision, using a mean squared error (MSE) training loss with a learning rate of 0.0001 and a batch size of 1.
}

\paragraph{BSMS-GNN}
\jess{
%
%
The BSMS-GNN also had an encoder-processor-decoder structure. The encoder only consisted of one MLP of input size 1 that represented the node type. When SV and DID were used, this increased to 11. Likewise, when FVF was implemented, it increased again to 12. It had two hidden layers with a hidden size of 128 and an output of size 128 as well. The decoder was similarly an MLP with an input of size 128, or 129 if FV was used, and had 2 hidden layers with hidden size 128 and an output size of 3 for the velocity and pressure fields.

The processor had an architecture similar to that of the Graph U-Net, with 9 different scales or levels and only 1 convolution per level. The number of nodes would decrease progressively from the bi-stride pooling method explained in the original paper. The input and output node and edge feature dimensions were 128 each, except when FVF was implemented, and the input node and edge features increased to 129 and 134, respectively. All activations used in the model were ReLU. The resulting number of trainable parameters in each model is shown in Table~\ref{tab:BSMS_params}.

Each model was trained on half-precision, using a mean squared error (MSE) training loss with a starting learning rate of 0.0001 that decayed exponentially at a decay rate of 0.7943 until the learning rate hit a minimum of 0.000001. All BSMS-GNN models were trained for 200 epochs with a batch size of 1.
}

\paragraph{Chen-GCNN}
The input node features to the baseline model consisted of a boolean representation of the geometry and the spatial coordinates. For models using IVE convolution, there were also input edge features, which were simply the average of the node features of each node in the edge. Thus, the input node features and input edge features both have a size of 3. When geometry features were used, SV and DID replaced the boolean representation, making the feature sizes 12 instead.

The model itself had eight convolution and smoothing layers, and finally a $1\times1$ convolution output layer. When IVE convolutions were used, the intermediate edge and node feature dimensions in each layer were $(4,8,16,32,64,64,32,16)$ and $(8,16,32,64,64,32,16,8)$ respectively, as they were in the original work. When SAGE convolutions were used, there were no intermediate edge feature dimensions. Instead, the number of filters used in the convolutions was fixed at $3$, and the node feature dimensions remained $(8,16,32,64,64,32,16,8)$ as in the IVE convolutions.
The final output layer produced $3$ node features representing predicted velocity and pressure. There are also skip connections from the input graph to the output of every convolutional and smoothing block, where the spatial coordinates of the nodes were concatenated to the node features. This is similar to the node attributes concatenated to the node features in the convolutions using FVF, although the cell volume was not used as a node attribute in the original work.

Table~\ref{tab:chengcnn_params} shows the number of trainable parameters in each model. The model depth and width were kept the same as the original Chen-GCNN model~\citep{chen2021} 
as it was already relatively small.
\begin{table*}[t!]
\centering
 \begin{tabular}{@{}c|c|c@{}}
 \toprule
 \textbf{Models}               &\textbf{Conv type}                           &\textbf{Number of trainable parameters} \\ \midrule
 Chen-GCNN (Baseline)          &\multirow{3}{*}{\begin{tabular}[c]{@{}c@{}}Invariant\\ Edge\\ Convolution\end{tabular}} &217853 \\
 Chen-GCNN w/ Geo        &   &222461 \\
 Chen-GCNN w/ FVF w/ Geo &  &227185 \\ \midrule \midrule
 Chen-GCNN (Baseline)          &\multirow{3}{*}{\begin{tabular}[c]{@{}c@{}}SAGE\\ Convolution\end{tabular}}             &20193  \\
 Chen-GCNN w/ Geo        &   &20337  \\
 Chen-GCNN w/ FVF w/ Geo &  &58067  \\ \bottomrule
\end{tabular}%
\caption{Sizes of the Chen-GCNN models compared.}
\label{tab:chengcnn_params}
\end{table*}

\begin{table*}[htb!]
\centering
 \begin{tabular}{@{}c|c|c@{}}
 \toprule
 \textbf{Models}                 &\textbf{Conv type}                                                                      &\textbf{Number of trainable parameters} \\ \midrule
 Graph U-Net (Baseline)          &\multirow{3}{*}{\begin{tabular}[c]{@{}c@{}}SAGE\\ Convolution\end{tabular}}             &65820 \\
 Graph U-Net w/ Geo        &                                                                                        &66396 \\
 Graph U-Net w/ FVF w/ Geo &                                                                                        &160465 \\ \midrule \midrule
 Graph U-Net (Baseline)          &\multirow{3}{*}{\begin{tabular}[c]{@{}c@{}}Vanilla\\ Graph\\ Convolution\end{tabular}}              &107667  \\
 Graph U-Net w/ Geo        &                                                                                        &108243  \\
 Graph U-Net w/ FVF w/ Geo &                                                                                        &105201  \\ \bottomrule
\end{tabular}%
\caption{Sizes of the Graph U-Net models compared.}
\label{tab:gunet_params}
\end{table*}

In training, the loss function chosen was the mean absolute error (MAE). We implemented early stopping, where training terminated if the MAE on the validation set did not improve after 60 epochs, or if it reached the maximum of 1000 epochs. 
A batch size of 64 was used, and the initial learning rate and decay rate were both $0.002$ to match the original work, with the decay schedule also following the same formula. Further details can be found in the original codebase\footnote{\url{https://github.com/cfl-minds/gnn_laminar_flow}}. The only difference in our training scheme implementation would be that we used half-precision training, and clipped the gradients to $5$.

\paragraph{Graph U-Net}
Our implementation of the Graph U-Net was adapted from the original codebase\footnote{\url{https://github.com/Extrality/AirfRANS}}.
The baseline model started and ended with an MLP encoder and decoder of layer sizes [7, 64, 64, 8] and [8, 64, 64, 4] respectively. The models that used geometry features had an encoder with layer sizes [16, 64, 64, 8] instead, as the SDF input feature was replaced with SV and DID.

For the U-Net itself, the models that did not use FVF both followed the original Graph U-Net architecture exactly. The downward and upward passes had five scales each. At each scale of the downward pass, the number of node features doubled, while the graph nodes were down-sampled by half to create radius graphs of radii 5 cm, 20 cm, 50 cm, 1 m and 10 m respectively. In the upward pass, skip connections concatenated the node feautures of the respective scale in the downward pass to that of the previous scale.
In models that used FVF, however, no down-sampling or up-sampling was done to preserve the mesh structure and its corresponding mesh characteristics. Nevertheless, the depth of the model and the presence of skip-links remained the same.

Table~\ref{tab:gunet_params} shows the number of trainable parameters in each model. Models that used SAGE convolutions kept the original Graph U-Net model width of 8~\citep{airfrans}, causing the model sizes to increase from the use of SV, DID, and FVF. This is to be consistent with the experimental practices of ~\citet{airfrans}, who allowed their model sizes to vary.
On the other hand, when vanilla graph convolutions were used, the width of the models without FVF was increased to 15 to make the model sizes similar to the model with FVF. This is because the models with only vanilla graph convolution without FVF were too small to train effectively when using the original model parameters.

\begin{table*}[htb!]
\centering
 \begin{tabular}{@{}c|c|c|c@{}}
 \toprule
 \textbf{Models}                                 &\textbf{Conv type}   &                                                               \textbf{ Model width}    &\textbf{\#Trainable parameters} \\ \midrule
 CFDGCN (Baseline)                               &\multirow{4}{*}{\begin{tabular}[c]{@{}c@{}}Vanilla\\ Graph\\ Convolution\end{tabular}}  &512 &1057283\\
 CFDGCN w/ Geo                       &      &512                                                                                        &1061891 \\
 CFDGCN w/ FVF w/ Geo                &      & 284                                                                                       &1021780 \\ 
 CFDGCN w/ Residual Training w/ FVF w/ Geo &    & 284                                                                                   &1021780 \\ \midrule \midrule
 CFDGCN (Baseline)                               &\multirow{4}{*}{\begin{tabular}[c]{@{}c@{}}SAGE\\ Convolution\end{tabular}}             &512 &2112003 \\
 CFDGCN w/ Geo                            &  &512                                                                                        &2121219 \\
 CFDGCN w/ FVF w/ Geo                   &   & 315                                                                                       &2153241 \\ 
 CFDGCN w/ Residual Training w/ FVF w/ Geo & &315                                                                                        &2153241 \\ \bottomrule
\end{tabular}%
\caption{Sizes of the CFDGCN models compared.}
\label{tab:cfdgcn_params}
\end{table*}
The learning rate for the experiments was set with a \jess{one-cycle cosine rate capped at 0.001}, with the number of epochs fixed at 1600 as in the scarce data regime of the original work~\citep{airfrans}. Half precision was used, and gradients clipped to 5.

For models that did not use FVF, mesh node-based graphs were used. 
However, the total number of nodes was fixed at 32000 via random subsampling of the input. Edges were formed between nodes within 5 cm of each other, with each node having a maximum of 64 neighbours. In training, each input graph was only subsampled once, and the training loss found against the subsampled ground-truth. In testing, subsampling was done repeatedly till a value was found for all the nodes in the original graph, and nodes with multiple values were assigned the average, before the test losses found against the unsampled ground-truth.
On the other hand, on FVF models, an unsampled cell centroid-based graph was used in both training and testing. As before, this was to maintain the mesh's structure and characteristics represented in the FVF.

\paragraph{CFDGCN}
As explained earlier,
to adapt the CFDGCN model to become generalisable to different geometries, we used the coarse mesh of each flow scenario in the data as a fixed input to the model, rather than as a trainable feature of the model. On all other aspects, however, the architecture of the CFDGCN was largely preserved.

The coarse mesh was upsampled once using squared distance-weighted, k-nearest neighbours interpolation to size of the fine mesh. The input node features of the baseline model were the spatial coordinates, angle of attack, mach number and SDF. If geometry features were used, the SDF was replaced by SV and DID. The graph was passed through 3 graph convolutions before the upsampled coarse mesh was appended to the output of the 3rd layer, and another 3 convolutions was performed to generate the final prediction.

All convolutions in models that did not use FVF had 512 hidden channels, just as the original CFDGCN did~\citep{cfdgcn}. However, the number of hidden channels was adjusted for the FV and residual models to keep model size similar for comparability. Table~\ref{tab:cfdgcn_params} shows the number of trainable parameters in each model. 
A batch size of $1$ was used on all experiments, and half precision was used. All other training parameters were kept the same as the original work. 
More details can be found from the codebase\footnote{\url{https://github.com/locuslab/cfd-gcn}}. 

\section{Coarse AirfRANS Dataset}
AirfRANS~\cite{airfrans} provides 1000 simulated airfoil cases but with a consistent mesh resolution. Having simulations of the same airfoil geometry but from a lower resolution is important for learning tasks such as super-resolution, which is needed for CFDGCN \cite{cfdgcn}. Since the residual training scheme is intended, albeit not limited, to demonstrate superior performance on super-resolution task, a lower-resolution variant of the AirfRANS cases were needed. The AirfRANS authors released the mesh generation script, which made it possible.

The majority of settings in the original AirfRANS dataset remain unchanged, except for the mesh resolution. In this modified dataset, the number of cells in all directions has been reduced to one-quarter of the original settings. Additionally, the gradings of the cell expansion have been reset to ensure a smooth transition of the cell thickness as listed in Table \ref{tab:coarse_grading}. 
\begin{table}[htb!]
    \centering
    \begin{tabular}{l|c}
    \toprule
    \textbf{Cell grading}	& \textbf{New values} \\
    \midrule
    yGrading	& 1000 \\
    yUGrading	& 1000 \\
    yDGrading	& 1000 \\
    xUUGrading	& 10 \\
    xDUGrading	& 10 \\
    xUMAeroGrading	& 2 \\
    xDMAeroGrading	& 2 \\
    xMGrading	& 1 \\
    xDTrailGrading	& 0.0001 \\
    xUDGrading	& 0.5 \\
    xDDGrading	& 0.5 \\
    leadUGrading	& 0.05 \\
    leadDGrading	& 0.05 \\
    \bottomrule
    \end{tabular}
    \caption{The new values of different cell grading for coarse AirfRANS dataset.}
    \label{tab:coarse_grading}
\end{table}

\section{Computational Environment} All experiments are run on a server with 32 Intel(R) Xeon(R) Silver 4110 CPU @ 2.10GHz processors, 256 GB RAM, and 3 dedicated Nvidia V100 GPU cards each with 32GB memory. All models are trained on a single GPU.

\section{Additional Results}

\subsection{Secondary Error Metrics}
\jess{
Due to the limited space available, the performance evaluation using Graph U-Net on the relative error and Spearman's correlation metrics are shown here in Table~\ref{tab:sub_second}. Note that for correlation, the best performing model achieves the closest correlation to 1.
}
\begin{table*}
\centering
\begin{tabular}{@{}c|c|cc|cc@{}}
\toprule
\multirow{2}{*}{\textbf{Models}} &
  \multirow{2}{*}{\textbf{\begin{tabular}[c]{@{}c@{}}Conv\\ type\end{tabular}}} &
  \multicolumn{2}{c|}{\textbf{\begin{tabular}[c]{@{}c@{}}Relative \\ error\end{tabular}}} &
  \multicolumn{2}{c}{\textbf{\begin{tabular}[c]{@{}c@{}}Spearman’s\\ correlation\end{tabular}}} \\
 &
   &
  $C_D$ &
  $C_L$ &
  $\rho_D$ &
  $\rho_L$ \\ \midrule
Graph U-Net (Baseline) &
  \multirow{3}{*}{\begin{tabular}[c]{@{}c@{}}SAGE\\ Convolution\end{tabular}} &
  5.315 &
  \textbf{0.211} &
  0.1151 &
  0.992 \\
Graph U-Net w/ Geo &
   &
  3.794 &
  0.279 &
  -0.0667 &
  \textbf{0.995} \\
Graph U-Net w/ FVF w/ Geo &
   &
  \textbf{0.949} &
  0.855 &
  \textbf{0.495} &
  0.992 \\ \midrule \midrule
Graph U-Net (Baseline) &
  \multirow{3}{*}{\begin{tabular}[c]{@{}c@{}}Vanilla\\ Graph\\ Convolution\end{tabular}} &
  28.270 &
  \textbf{0.192} &
  \textbf{0.109} &
  0.969 \\
Graph U-Net w/ Geo &
   &
  27.671 &
  0.208 &
  0.084 &
  \textbf{0.985} \\
Graph U-Net w/ FVF w/ Geo &
   &
  \textbf{2.474} &
  0.519 &
  -0.191 &
  0.962 \\ \bottomrule
\end{tabular}%
\caption{Graph U-Net performance evaluated using relative error and Spearman's correlation metrics.}\label{tab:sub_second}
\end{table*}
\jess{
The inferior performance on some metrics such as the volume MSE of $\nu_t$ and relative error of $C_L$ and $\rho_L$ is attributed to the loss function introduced in the work of \citet{airfrans}, which is not designed to explicitly minimise these metrics.
}

\subsection{Effectiveness of Residual Training}
Due to the limited cases demonstrating the effectiveness of residual training, additional experiments using the AirfRANS dataset were done with the MGN model, and Chen-GCNN model with IVE convolution. The results are shown in Table~\ref{tab:chengcnn_airfrans} and Table~\ref{tab:MGN_res} respectively.

\begin{table}[t!]
\centering
 \begin{tabular}{@{}c|c|c@{}}
\toprule
\multirow{2}{*}{\textbf{Models}} &
  \multirow{2}{*}{\textbf{\begin{tabular}[c]{@{}c@{}}Conv \\ type\end{tabular}}} &
  \multirow{2}{*}{\textbf{\begin{tabular}[c]{@{}c@{}}MSE \\ $\times 10^{-2}$\end{tabular}}} \\ &  & \\ \midrule
Chen-GCNN w/ FVF w/ Geo &
  \multirow{3}{*}{IVE} &
  5.6104 \\
Chen-GCNN w/ Residual Training 
&  &\multirow{2}{*}{4.2166} \\
w/ FVF w/ Geo & & \\
\bottomrule
\end{tabular}%
\caption{Performance evaluation using Chen-GCNN on Airfrans dataset.}
\label{tab:chengcnn_airfrans}
\end{table}
\begin{table}[htb!]
\centering
\begin{tabular}{c|c}
\toprule 
\textbf{Models} & \textbf{MSE} $\times10^-2$ \\ \midrule 
MGN w/ FVF w/ Geo & 3.3811 \\
MGN w/ Residual Training  &\multirow{2}{*}{1.8855} \\
w/ FVF w/ Geo & \\
\bottomrule 
\end{tabular}
\caption{Performance evaluation using MGN.}
\label{tab:MGN_res}
\end{table}

The unlike the Chen-GCNN models trained on the 2DSHAPES dataset, these models used a batchsize of 1. The training parameters of the MeshGraphNet models remains as before.
Unlike in the case with CFDGCN, the original works did not assume the availability of the coarse solutions to use as estimates, and so this had to be considered an additional assumption. Also, while the baseline CFDGCN model already took the coarse fields as inputs, the baseline MeshGraphNet and Chen-GCNN models did not, and so the coarse fields had to be additionally included in the inputs for residual training.
\subsection{Extrapolation and Full Data Regimes}

While the experiments using the AirfRANS dataset were done using the scarce data regime, more CFDGCN models were trained using the Reynold's Number and AoA extrapolation regimes and the full data regime. The results are shown in Table~\ref{tab:cfdgcn_reynolds}. They demonstrate that the proposed methods remain effective under a variety of scenarios with different data availability.
\begin{table*}[htb!]
\centering
\begin{tabular}{c|c|c}
\toprule
\textbf{Models} & \textbf{\begin{tabular}[c]{@{}c@{}}Data\\Regime\end{tabular}} & \textbf{\begin{tabular}[c]{@{}c@{}}MSE\\ $\times10^{-2}$\end{tabular}} \\ \midrule
CFDGCN (Baseline) & \multirow{4}{*}{\begin{tabular}[c]{@{}c@{}}Reynold's\\Number\end{tabular}} & 0.3210 \\
CFDGCN w/ Geo &  & 0.0889 \\
CFDGCN w/ FVF w/ Geo &  & 0.0678 \\
CFDGCN w/ Residual Training w/ FVF w/ Geo &  & \textbf{0.0661} \\ \midrule \midrule
CFDGCN (Baseline) & \multirow{4}{*}{\begin{tabular}[c]{@{}c@{}}AoA\end{tabular}} & 0.9098 \\
CFDGCN w/ Geo &  & 0.1846 \\
CFDGCN w/ FVF w/ Geo &  & 0.1772 \\
CFDGCN w/ Residual Training w/ FVF w/ Geo &  & 	\textbf{0.1595} \\ \midrule \midrule
CFDGCN (Baseline) & \multirow{4}{*}{\begin{tabular}[c]{@{}c@{}}Full\end{tabular}} & 0.0737 \\
CFDGCN w/ Geo &  & 0.0668 \\
CFDGCN w/ FVF w/ Geo &  & 0.0655 \\
CFDGCN w/ Residual Training w/ FVF w/ Geo &  &  \textbf{0.0499} \\ \bottomrule
\end{tabular}
\caption{Extrapolation and full data results of the CFDGCN models with vanilla graph convolution}
\label{tab:cfdgcn_reynolds}
\end{table*}
%
\begin{figure*}[t!]
     \centering
     \begin{subfigure}[b]{\columnwidth}
         \centering
    \includegraphics[width=\textwidth]{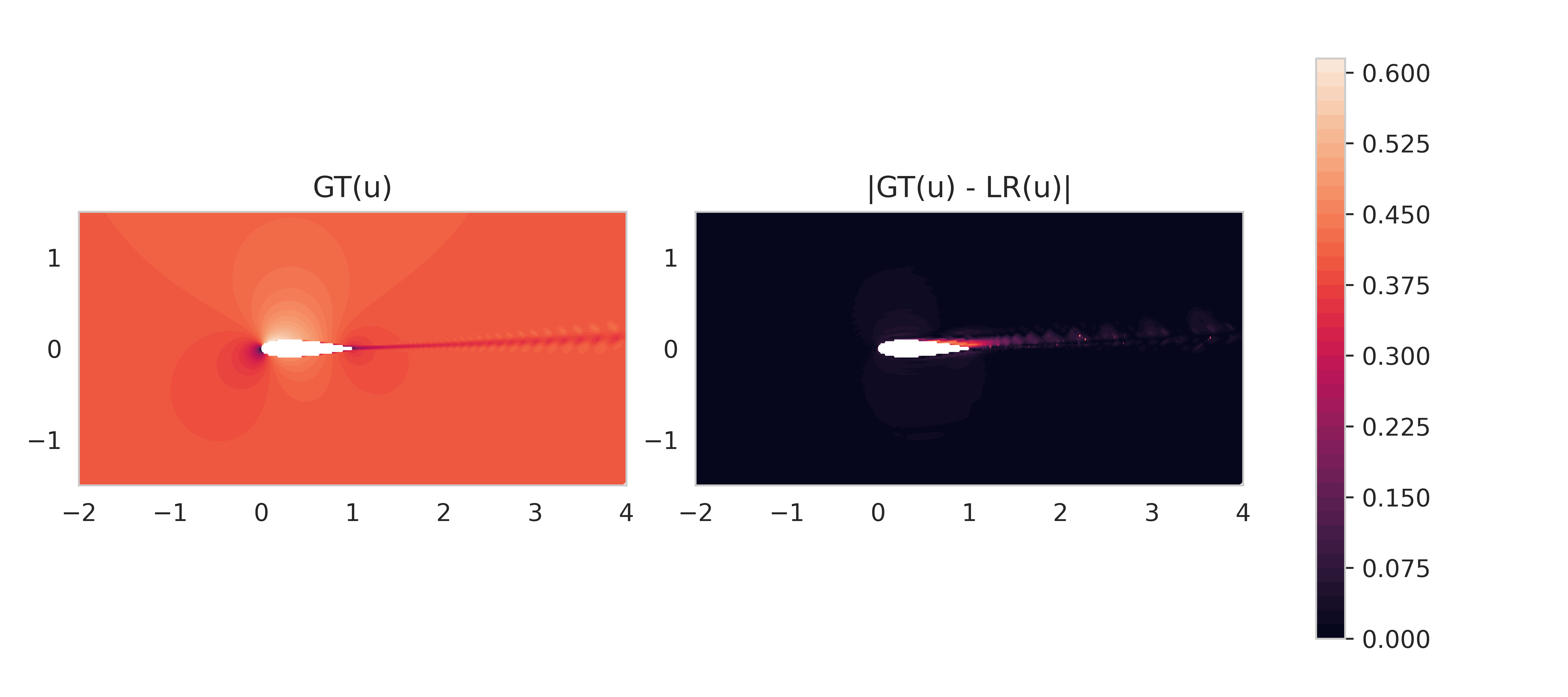}
         \caption{Velocity x-component (u)}
         \label{fig:res1}
     \end{subfigure}
     \hfill
     \begin{subfigure}[b]{\columnwidth}
         \centering
         \includegraphics[width=\textwidth]{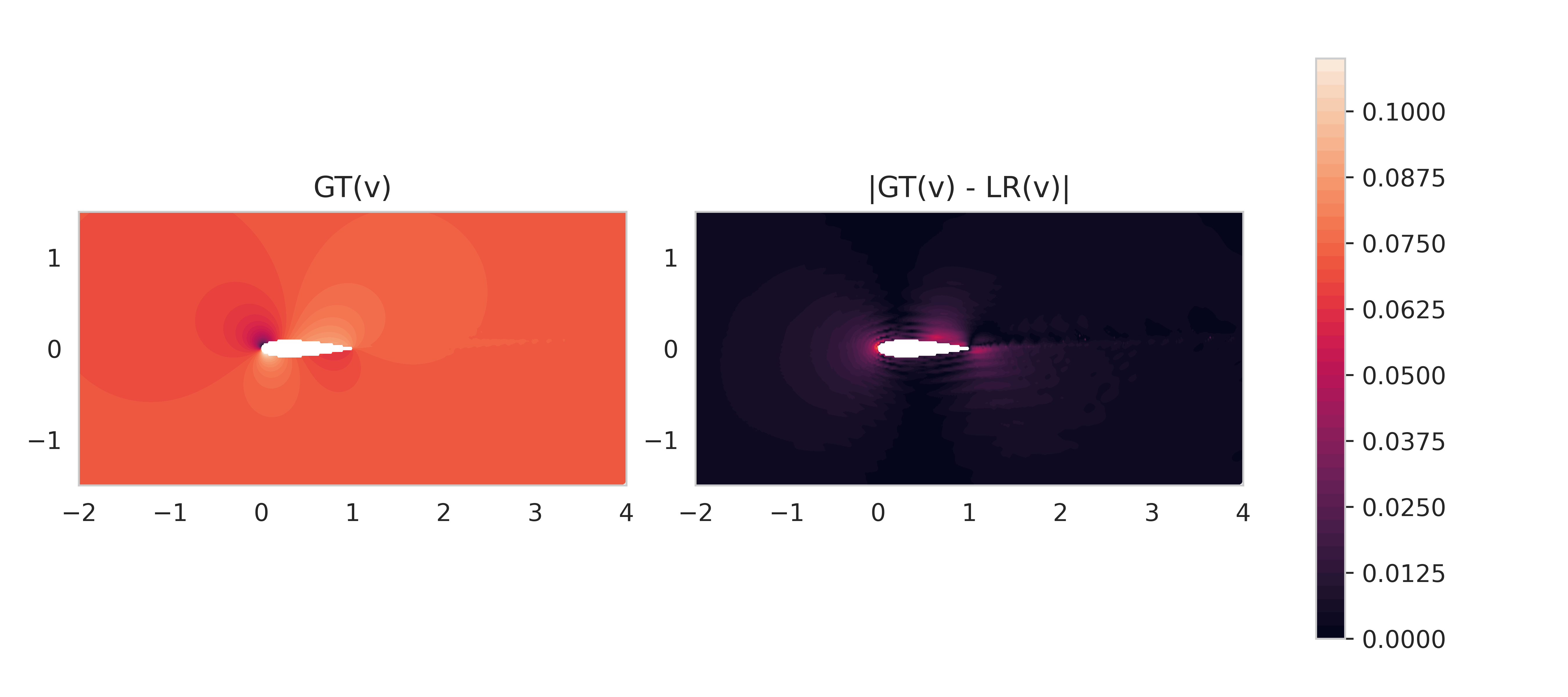}
         \caption{Velocity y-component (v).}
         \label{fig:res2}
     \end{subfigure}
     \hfill
     \begin{subfigure}[b]{\columnwidth}
         \centering
         \includegraphics[width=\textwidth]{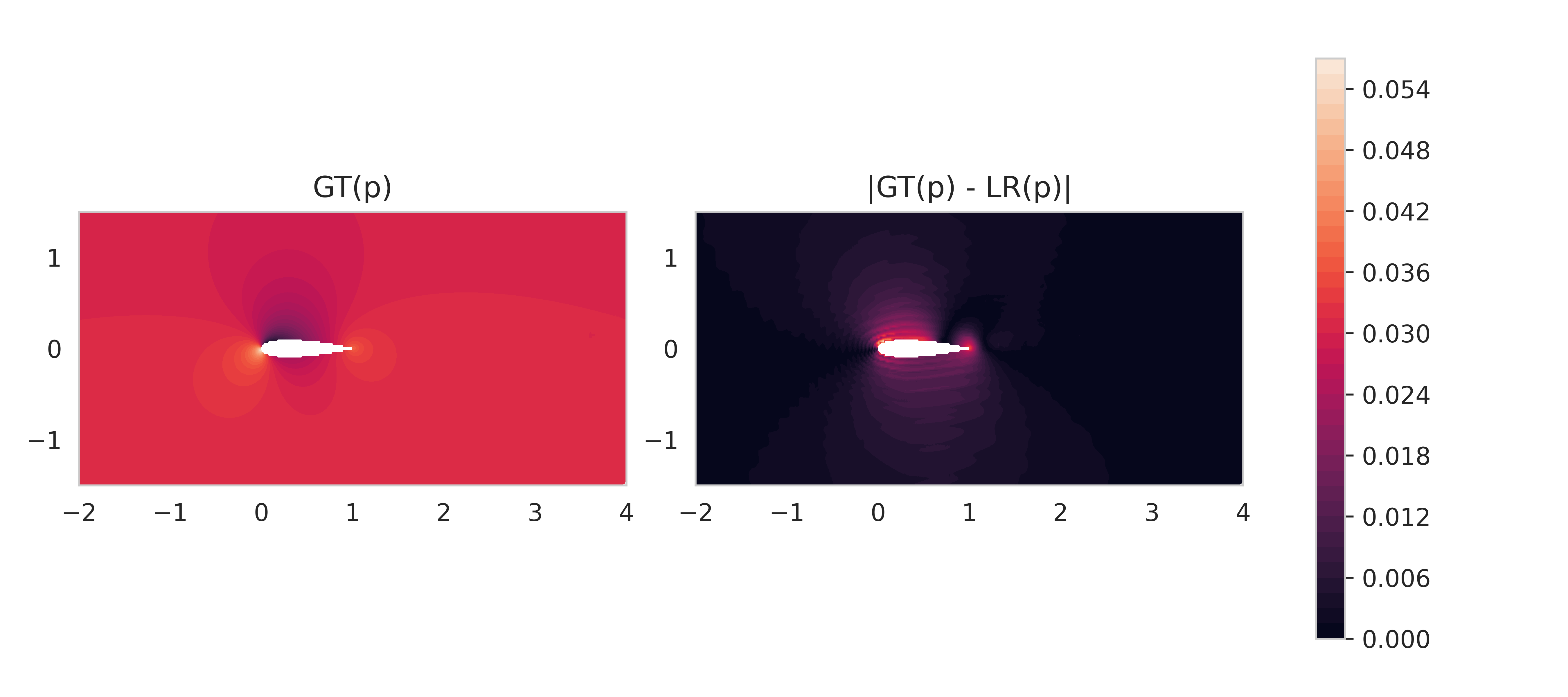}
         \caption{Pressure (p).}
         \label{fig:res3}
     \end{subfigure}
        \caption{An example airfoil from AirfRANS shows that (a) velocity x-component, u remains invariant in majority regions (left), hence the absolute value of residual ($|GT(u) - Upsample(LR(u))|$) is close to 0 in those regions (right). As a result, the model has less difficulty learning flow field values in these regions easing the learning process. We observe a similar phenomenon in other flow variables in (b) and (c). }
    \label{fig:res}
\end{figure*}
\subsection{Rationale for residual training}
As mentioned earlier, 
the residual field helps the model to focus on the more nuanced areas where the low-resolution fields tend to be inaccurate. In Figure~\ref{fig:res}, we illustrate this phenomenon to emphasize the importance of residual training. We show ground truth velocity components (u,v) and pressure (p) corresponding to the ground truth flow field on the left, while the corresponding residual fields are on the right. In most of the regions, the residual is close to zero. Hence, the model has less difficulty learning flow field values in these regions, which are coloured in darker shades. 

\begin{figure*}[t!]
     \centering
     \begin{subfigure}[b]{\columnwidth}
         \centering
\includegraphics[width=\textwidth]{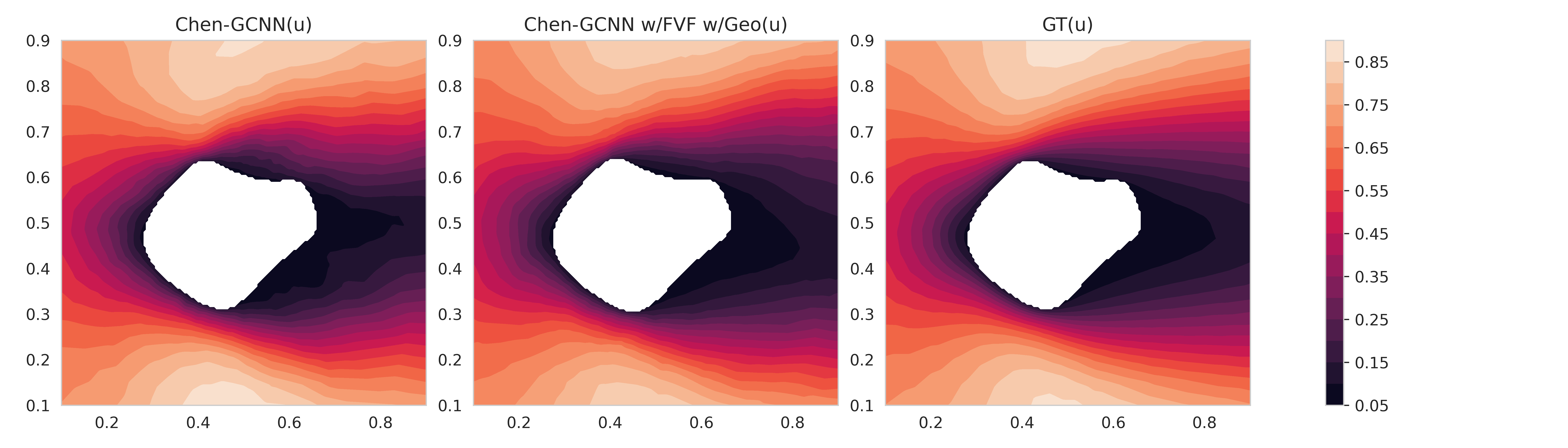}
         \caption{Velocity x-component (u) visualisation.}
         \label{fig:vis1}
     \end{subfigure}
     \begin{subfigure}[b]{\columnwidth}
         \centering
         \includegraphics[width=\textwidth]{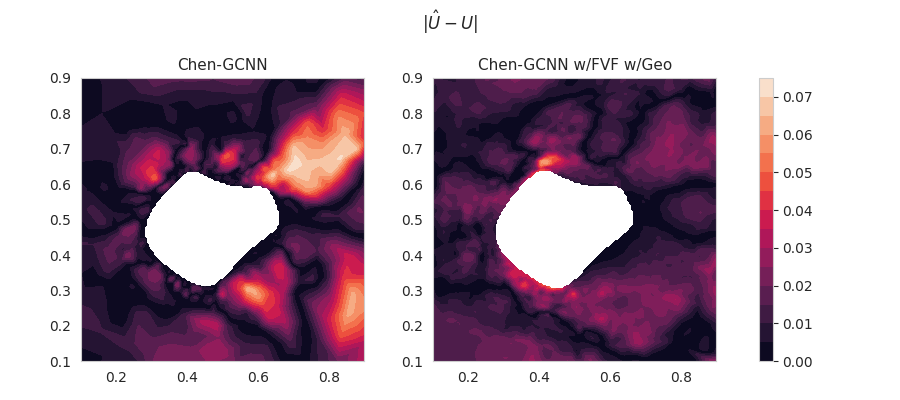}
          \caption{Absolute error of predictions.}
         \label{fig:vis2}
    \end{subfigure}
    \caption{(a)Visual comparison between Chen-GCNN (baseline) and Chen-GCNN w/ FVF w/ Geo: The velocity x-component ($u$) from baseline, Chen-GCNN w/ FVF w/ Geo and the ground truth (from left to right). (b) Absolute error comparison wrt. the ground truth: Chen-GCNN w/ FVF w/ Geo has a comparatively smaller high-error ($[0.050,0.075]$) region than the baseline. Here the darkest shade indicates a low-error region, and the lightest shade indicates a high-error region. }
    \label{fig:vis}
\end{figure*}
\begin{figure*}[t!]
\centering 
\includegraphics[width=0.8\textwidth]{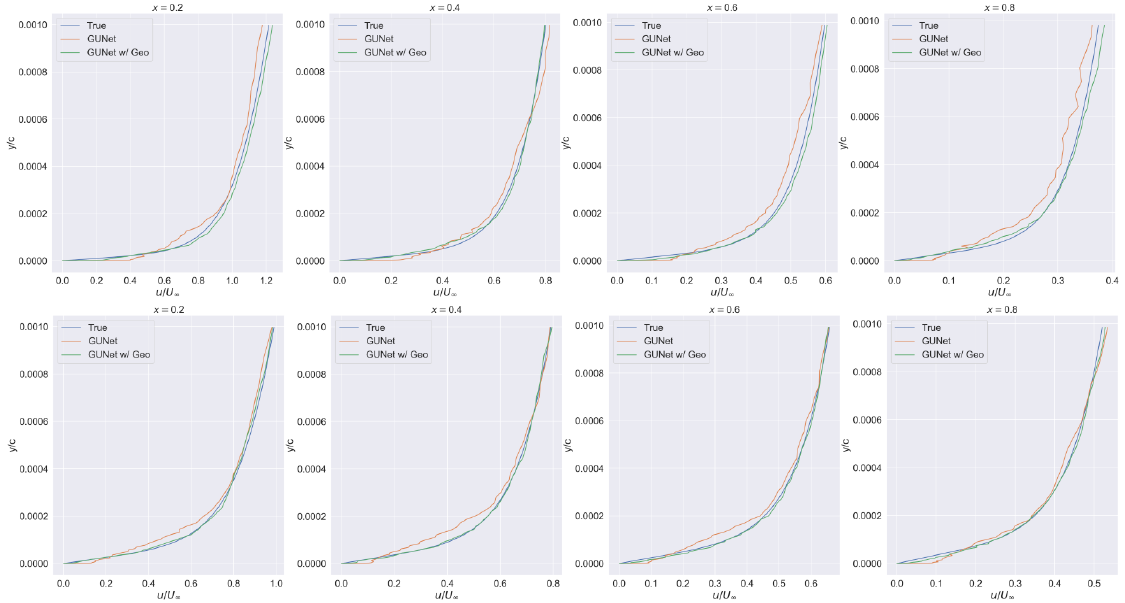}
\caption{Comparison of the predicted boundary layers profiles on two random test airfoils at four different abscissas in the scarce data regime with respect to the true ones. Each row of plots
represents a different airfoil, and each column of plots represents a different abscissa. The $x$
component of the velocity is denoted by u, and the inlet velocity is denoted
by $U_\infty$.}
\label{fig:bl_profile}
\end{figure*}

\section{Predictive Error Visualisation} 
Figure~\ref{fig:vis1} shows a comparison between the predicted velocity x-component flow fields, $u$, of the baseline method Chen-GCNN and Chen-GCNN w/ FVF w/ Geo. Both models used invariant edge convolutions.
We observe that the predictions of both methods capture the flow features present in the Ground truth, GT($u$). A close inspection of the absolute error of the predicted flow field in figure~\ref{fig:vis2} reveals that Chen-GCNN w/ FVF w/ Geo produces relatively small regions with high error, whereas there are many regions with high error in the flow field predicted by Chen-GCNN. This observation indicates that incorporating our proposed features reduces the predictive error across different regions in the domain.

\section{Boundary Layer Profile} We analyse the velocity profile at the boundary layer in Figure~\ref{fig:bl_profile}. We plot the normalised distance from the airfoil boundary ($y/c$) vs the normalised velocity x-component $u/U_{\infty}$. Figure~\ref{fig:bl_profile} shows that GUNet w/ Geo predicts velocity $x$-component near the airfoil boundary more accurately compared to the baseline GUNet. 